\documentclass{article}
\usepackage{arxiv}
\usepackage[utf8]{inputenc}
\usepackage[T1]{fontenc}    
\usepackage{booktabs}       
\usepackage{nicefrac}       
\usepackage{microtype}      
\usepackage{xcolor}         
\usepackage{bm}
\usepackage{url}
\usepackage{amsmath}
\usepackage{amsthm}
\usepackage{amsfonts}
\usepackage{amssymb}
\usepackage{cancel}
\usepackage{graphicx}
\usepackage{natbib}
\usepackage{multirow}
\usepackage{makecell}
\usepackage{enumitem}
\usepackage{tabu}
\usepackage{wrapfig}
\usepackage{footmisc}
\usepackage{subcaption}

\usepackage{algorithm}
\usepackage{algpseudocode}

\usepackage{hyperref}

\newtheorem{thm}{Theorem}
\newtheorem{prop}{Proposition}

\newtheorem{dfn}{Definition}
\newtheorem{remark}{Remark}

\newcommand{\Tr}{\text{\normalfont Tr}}
\newcommand{\Cov}{\text{\normalfont Cov}}
\newcommand{\Var}{\text{\normalfont Var}}
\newcommand{\diag}{\text{\normalfont diag}}
\newcommand{\X}{\bm{X}}
\newcommand{\Xs}{\bm{X^*}}
\newcommand{\xv}{\bm{x}}
\newcommand{\xpv}{\bm{x'}}
\newcommand{\xsv}{\bm{x^*}}

\newcommand{\zv}{\bm{z}}
\newcommand{\zpv}{\bm{z'}}
\newcommand{\xvi}{\bm{x_i}}
\newcommand{\xvj}{\bm{x_j}}

\newcommand{\N}{\mathcal{N}}

\newcommand{\Oo}{\mathcal{O}}
\newcommand{\Ff}{\mathcal{F}}
\newcommand{\spn}{\text{\normalfont span}}
\newcommand{\nv}{\bm{0}}
\newcommand{\R}{\mathbb{R}}
\newcommand{\E}{\mathbb{E}}
\newcommand{\I}{\bm{I}}
\newcommand{\Ss}{\bm{S}}
\newcommand{\A}{\bm{A}}
\newcommand{\K}{\bm{K}}
\newcommand{\Ks}{\bm{K^*}}
\newcommand{\FC}{\text{\normalfont FC}}
\newcommand{\EFC}{\text{\normalfont EFC}}
\newcommand{\Kb}{\bm{\bar{K}}}
\newcommand{\Kttr}{\bm{K}_{DT}}
\newcommand{\Knntr}{\bm{K}_{kNN}}
\newcommand{\kb}{\bar{k}}
\newcommand{\yv}{\bm{y}}

\newcommand{\fhv}{\bm{\hat{f}}}

\newcommand{\fh}{\hat{f}}
\newcommand{\thetahv}{\bm{\hat{\theta}}}

\newcommand{\lambdab}{\bar{\lambda}}
\newcommand{\phiv}{\bm{\varphi}}
\newcommand{\Ph}{\bm{\Phi}}
\newcommand{\muv}{\bm{\mu}}

\newcommand{\po}{\text{\normalfont po}}

\newcommand{\ysv}{\bm{y^*}}
\newcommand{\fhsv}{\bm{\hat{f}^*}}
\newcommand{\ns}{n^*}
\newcommand{\Sss}{\bm{S^*}}
\newcommand{\thetatv}{\bm{\tilde{\theta}}}

\title{A Rigorous, Tractable Measure of Model Complexity}

\author{%
  Oskar Allerbo\\
  Department of Mathematics\\
  KTH Royal Institute of Technology\\
  \texttt{oallerbo@kth.se} \\
  \And
  Thomas B. Schön\\
  Department of Information Technology \\
  Uppsala University\\
  \texttt{thomas.schon@uu.se}\\
}

\begin{document}
\maketitle
\begin{abstract}
An accurate assessment of a model's complexity is crucial for topics such as interpretation, generalization, and model selection. However, existing complexity measures either rely on heuristic assumptions or are computationally prohibitive. In this paper, we present a mathematically rigorous yet easy-to-compute measure of model complexity that is based on the similarities between feature vectors for different inputs. It is thus well-defined for any parametric model (via NTK feature generalization) and for kernel-based non-parametric models (via the induced kernel). We prove that our measure of complexity generalizes model-specific complexity measures, including polynomial degree (for polynomial regression), kernel length scale (for Matérn kernels), number of neighbors (for k-nearest neighbors), number of splits (for decision trees), and number of trees (for random forests). We also use our measure to obtain new insights into the double descent phenomenon for random Fourier features, random forests, neural networks, and gradient boosting.
\end{abstract}

\section{Introduction}
%
The ability to quantify model complexity lies at the heart of machine learning. Complexity influences how models generalize from data, how they can be interpreted, and how they should be selected among competing alternatives. Furthermore, as the computational and environmental costs of machine learning continue to grow, complexity measurements will likely be an important tool for the design of energy-efficient learning systems. Yet, despite its central role, there is no canonical, model-agnostic way to gauge a model's complexity. While simple heuristics, such as the number or magnitude of parameters, yield very crude estimates, hyperparameter-based approaches, such as polynomial degree or kernel length scale, do not generalize across model classes. More rigorous methods, including the Vapnik-Chervonenkis dimension (VCD) \citep{vapnik2013nature}, Rademacher complexity (RMC) \citep{bartlett2002rademacher}, and effective number of parameters (or effective degrees of freedom, ENP) \citep{efron1986biased}, are difficult, or even impossible, to compute in practice, leaving the user to resort to crude bounds and/or approximations.
The topic is further complicated by the often-overlooked distinction between \emph{model} and \emph{function} complexity, with the former setting a ceiling on the latter. However, a \emph{complex model} can always generate a \emph{simple function}, e.g.\ by setting all parameters to zero. 

In this paper, we introduce a \textbf{model-agnostic} complexity measure that is \textbf{mathematically rigorous} yet \textbf{easy-to-compute}. We base our measure on the \textbf{similarities between feature vectors}, $\phiv(\xvi)\in\R^p$, across inputs $\xvi\in\R^d$. 
This is motivated by the following simple observation: \emph{When the features across inputs are very similar, the model is more constrained than when the features are more independent}. As an example, in the extreme case when $\phiv(\xvi)=\phiv(\xvj)$ for all $i,j$, the inferred function is independent of $\xv$ and thus a constant, 
i.e., the model is extremely simple. On the other hand, if all features are independent across data, the prediction at $\xvi$ is not restricted by that at $\xvj$, i.e., the model is extremely flexible.

Our \textbf{main contribution} is the introduction and analysis of a new feature-based model complexity measure. Additionally, we
\begin{itemize}[leftmargin=6mm]
\item 
show how it measures \emph{model}, rather than \emph{function}, complexity.
\item
analyze it for several different constant-feature models and prove that it is well-behaved with respect to model hyperparameters.
\item
use it to provide new insights into the double descent phenomenon.
\end{itemize}

All proofs are deferred to Appendix~\ref{sec:proofs}, and the code is available at \url{http://github.com/allerbo/feature_complexity}.

\section{Related Work}
\textbf{Model and Function Complexity:}
Commonly used complexity measures are the \emph{Vapnik-Chervonenkis dimension} \citep{vapnik2013nature}, the \emph{Rademacher complexity} \citep{bartlett2002rademacher,bartlett2005local}, and the \emph{effective number of parameters} \citep{efron1986biased,ye1998measuring,efron2004estimation}. One important, but often overlooked, distinction is that while the VCD and RMC measure the \emph{model complexity} by considering all possible functions of a given model class, the ENP measures the complexity of a \emph{single learned function}. This makes the first two very difficult, if not impossible, to calculate in practice, while the last may provide a poor proxy for the actual model complexity.

The \emph{VCD} and \emph{RMC} have been used as complexity metrics in many different settings, including robust learning \citep{yin2019rademacher,depersin2024robust,zhou2026rademacher}, kernel learning \citep{tan2004support, cortes2013learning,xavier2025early}, decision trees \citep{kaariainen2004selective,aslan2009calculating}, deep learning \citep{truong2022rademacher,xiao2022adversarial,yang2023nearly,sepliarskaia2024vc}, and reinforcement learning \citep{duan2021risk,xie2024vc}. 
The \emph{ENP} tends to work best for models with constant features but has also been used for, e.g., neural networks \citep{gao2016degrees,zhou2026measuring}. Since it is calculated using only in-sample data, its usefulness is limited in some situations; see, e.g.\ \citet{janson2015effective}. Generalizations that also include out-of-sample data have been proposed by \citet{rosset2020fixed}, \citet{luan2021predictive}, \citet{curth2023u}, and \citet{patil2024revisiting} and have been used for, e.g., random forests and neural networks \citep{patil2024revisiting,curth2024random,jeffares2024deep}. 

For neural networks, many alternative types of complexity measures have been proposed, based on things like
the number of linear regions \citep{montufar2014number, hanin2019complexity, chen2022improved}, 
activation patterns and regions \citep{raghu2017expressive,tseran2021expected}, 
Fisher information \citep{liang2019fisher, di2025spectral}, 
topology \citep{bianchini2014complexity}, 
and norms of weight matrices \citep{bartlett2017spectrally,demoss2025complexity}.
However, apart from being restricted to neural networks, these methods often impose additional requirements on the activation functions, limiting their applicability.

\citet{nakkiran2020deep} defined the effective model complexity (EMC) as the maximum number of observations a model can interpolate, which is intuitive and model-agnostic, but often impossible to calculate in practice, especially for over-parameterized models.

\textbf{Double Descent:}
Double descent, the phenomenon that models eventually generalize better with increasing complexity, was named and popularized by \citet{belkin2019reconciling}, but the behavior had previously been observed, see e.g.  \citet{krogh1991simple,bartlett2002rademacher,zhang2017understanding,dziugaite2017computing,belkin2018overfitting}. The subject has been studied for many classes of models, including linear and logistic regression \citep{bartlett2020benign, hastie2022surprises, mei2022generalization, deng2022model}, k-nearest neighbors \citep{curth2024classical},
random projections \citep{bach2024high}, 
regression trees and boosting \citep{belkin2019reconciling, curth2023u}, principal component regression \citep{gedon2024no}, neural networks \citep{geiger2020scaling, nakkiran2021deep, ju2021generalization, jeffares2024deep},
and kernel regression \citep{poggio2019double, liu2021kernel}.\\
An obstacle to studying double descent is the lack of an easily computable, widely accepted measure of model complexity, which can make the phenomenon difficult to interpret.
Our complexity measure is both model-agnostic and easy to compute, and it can thus help to facilitate the interpretation of double descent.

\section{A New Measure of Model Complexity}
\label{sec:fc_intro}
We base our measure of model complexity on the similarity between its feature vectors, $\phiv(\xvi)$, evaluated at different inputs $\xvi$---or, equivalently, on the induced kernel, $k(\xvi,\xvj)=\phiv(\xvi)^\top\phiv(\xvj)$. Since we are interested in the degree of similarity between the features, but not in the sign of the interaction, we use the squared cosine similarity:

\begin{dfn}[Feature-Similarity Complexity, FC]
For a model $\Ff$, with features $\phiv(\xv)$, induced kernel $k(\xv,\xpv)=\phiv(\xv)^\top\phiv(\xpv)$, and normalized induced kernel $\kb(\xv,\xpv):=\frac{\phiv(\xv)^\top}{\|\phiv(\xv)\|_2}\frac{\phiv(\xpv)}{\|\phiv(\xpv)\|_2}= \frac{k(\xv,\xpv)}{\sqrt{k(\xv,\xv)\cdot k(\xpv,\xpv)}}$,
\begin{equation}
\label{eq:fc}
\begin{aligned}
\FC(\Ff):=&1-\E_{\xv,\xpv}\Bigg(\bigg(\frac{\phiv(\xv)^\top\phiv(\xpv)}{\|\phiv(\xv)\|_2\cdot \|\phiv(\xpv)\|_2}\bigg)^2\Bigg)=1-\E_{\xv,\xpv}\left((\kb(\xv,\xpv)^2\right).
\end{aligned}
\end{equation}
\end{dfn}
\begin{remark}
Since we identify a model through its features, or equivalently, its kernel, by construction $\FC(\Ff)=\FC(\phiv)=\FC(k)$. We will use these equivalent notations interchangeably.\\
Some non-parametric kernel models have kernels without corresponding feature expansions. In these cases, the FC is still well-defined through its kernel formulation. For consistency, we use the name \emph{feature-similarity complexity} also in these cases, although \emph{kernel-similarity complexity} is the technically correct term.
\end{remark}
\begin{remark}~\\
With $\alpha_\varphi(\xv,\xpv):=\cos^{-1}\left(\frac{\phiv(\xv)^\top\phiv(\xpv)}{\|\phiv(\xv)\|_2\cdot \|\phiv(\xpv)\|_2}\right)$ denoting the angle between $\phiv(\xv)$ and $\phiv(\xpv)$, Equation~\ref{eq:fc} becomes
\begin{equation*}
\label{eq:fc_alpha}
\begin{aligned}
\FC(\Ff)&=1-\E_{\xv,\xpv}(\cos^2\alpha_\varphi(\xv,\xpv))=\E_{\xv,\xpv}(\sin^2\alpha_\varphi(\xv,\xpv)).
\end{aligned}
\end{equation*}
Since $\sin(0)=0$ and $\sin(\pi/2)=1$, the FC measures the expected degree of orthogonality between the feature vectors, ignoring the sign of the angle. Hence, the complexity is large when the feature vectors are close to orthogonal, and small when the feature vectors are close to parallel.
\end{remark}

Constant-feature models can be written in the form $\fh(\xv,\thetahv)=\thetahv^\top\phiv(\xv)$, where $\thetahv$ denotes the model parameters. This includes, e.g., linear and polynomial regression, kernel regression, and neural networks with random features. In this case, the features act as a basis for the model, i.e., $\Ff=\spn\{\phiv\}$, and the parameters act as coordinates in this basis. Thus, by measuring the complexity of features rather than parameters, we measure \emph{model complexity} rather than \emph{function complexity}. However, a complex model can always produce a simple function, e.g., by heavy regularization, which affects the parameters/function but not the features/model. For a deeper discussion on the differences between model and function complexity, see Appendix~\ref{sec:fct_mod}.

For non-constant-feature models, where the features depend on the parameters and thus tend to change during training, we use the neural tangent kernel (NTK) \citep{jacot2018neural} generalization of features, $\phiv(\xv,\thetahv):=\nabla_{\thetahv}\fh(\xv,\thetahv)$. Although the NTK has primarily been studied in the context of neural networks, it is well-defined for any parametric model, and since $\nabla_{\thetahv}(\thetahv^\top\phiv(\xv))=\phiv(\xv)$, it indeed generalizes the constant-feature case. For non-constant-features, $\fh(\xv,\thetahv)=\thetahv^\top\phiv(\xv)$ does not hold globally. However, by definition we always have $d\fh(\xv,\thetahv)=\bm{d}\thetahv^\top\phiv(\xv,\thetahv)$, and thus $\phiv(\xv,\thetahv)$ acts as a local basis around $\thetahv$. Since the features now depend on the parameters and thus tend to change during training, so does the model complexity, as occurs, e.g., with neural networks and gradient boosting.\footnote{Since $\phiv(\xv,\thetahv)=\nabla_{\thetahv}\fh(\xv,\thetahv)$, we can also think of the FC as the similarities between gradients during training: When the gradients for different values of $\xv$ are very similar, the model has less freedom to model the data than when the gradients are more independent.}

Using the equivalent feature and kernel formulations, the FC is well-defined not only for parametric models but also for non-parametric kernel models, including those with infinite feature expansions. When $\fhv$ has multivariate outputs, e.g., for classification with multiple classes, we generalize the NTK features from gradients to Jacobians: $\phiv(\xv,\thetahv):=\text{\normalfont vec}(\bm{J}_{\thetahv}(\fhv(\xv,\thetahv)))$, where $\bm{J}_{\thetahv}$ denotes the Jacobian with respect to $\thetahv$ and $\text{\normalfont vec}$ denotes the vectorization operator. This corresponds to generalizing the dot product in Equation~\ref{eq:fc} to the Frobenius inner product, $\Tr(\bm{A}^\top\bm{B})=\text{\normalfont vec}(\bm{A})^\top\text{\normalfont vec}(\bm{B})$.\\

To calculate $\FC(\Ff)$, we need access to the true distribution of the inputs, which is not always available. However, given a dataset $\{\xvi\}_{i=1}^n$, we can evaluate the empirical FC by replacing the expectation in Equation~\ref{eq:fc} by the sample average:
\begin{dfn}[Empirical Feature-Similarity Complexity, EFC]
\begin{equation}
\label{eq:efc}
\begin{aligned}
&\EFC(\Ff, \{\xvi\}_{i=1}^n):=1-\frac{1}{n^2-n}\cdot\sum_{\substack{i,j=1\\i\neq j}}^n\underbrace{\left(\frac{\phiv(\xvi)^\top\phiv(\xvj)}{\|\phiv(\xvi)\|_2\cdot \|\phiv(\xvj)\|_2}\right)^2}_{=\kb(\xvi,\xvj)^2},
\end{aligned}
\end{equation}
\end{dfn}
where $\EFC(\Ff, \{\xvi\}_{i=1}^n)=\EFC(\phiv, \{\xvi\}_{i=1}^n)=\EFC(k, \{\xvi\}_{i=1}^n)$, just as for the FC.
\begin{remark}
\label{rmk:efc_n}
Since the EFC is an unbiased estimate of the FC, apart from random fluctuations, it is independent of the sample size $n$.
\end{remark}
The EFC is closely related to the entropy of the normalized kernel matrix $\Kb$, where $(\Kb)_{ij}=\kb(\xvi,\xvj)$, and the entropy of a positive definite matrix is computed as the entropy of its normalized eigenvalues (since the normalized eigenvalues sum to 1, we can treat them as a discrete probability distribution). According to Proposition~\ref{thm:lin_entr}, the EFC equals the normalized linear entropy\footnote{\label{fn:lin_entr}The linear entropy of a positive definite matrix is a linearization of the von Neumann entropy, where the latter is defined as the Shannon entropy of the normalized eigenvalues: $-\sum_{i=1}^n \lambdab_i\log(\lambdab_i)$ for $\lambdab_i:=\lambda_i/\sum_{i=1}^n\lambda_i$. For the linear entropy, $\log(x)$ is replaced by its first-order Taylor expansion at $x=1$, which is $x-1$, and it is thus given by $\sum_{i=1}^n \lambdab_i(1-\lambdab_i)$. The normalized versions of the entropies, which always obtain values in $[0,1]$, are $\frac{-1}{\log(n)}\sum_{i=1}^n\lambdab_i\log(\lambdab_i)$ and $\frac{n}{n-1}\sum_{i=1}^n \lambdab_i(1-\lambdab_i)$, respectively. The computational costs are $\Oo(n^3)$ for the von Neumann entropy and $\Oo(n^2)$ for the linear entropy.} of the corresponding kernel matrix, where a larger complexity corresponds to a higher entropy.

\begin{prop}~\\
\label{thm:lin_entr}
Let $\Kb\in\R^{n\times n}$ denote the normalized kernel matrix for the sample $\{\xvi\}_{i=1}^n$, i.e., $(\Kb)_{ij}= \kb(\xvi,\xvj).$ Then
\begin{equation*}
\begin{aligned}
\EFC(k,\{\xvi\}_{i=1}^n)=\overline{\mathcal{LE}}(\Kb),
\end{aligned}
\end{equation*}
where $\overline{\mathcal{LE}}$ denotes the normalized linear entropy.\footref{fn:lin_entr}
\end{prop}

\begin{figure*}[t!]
\center
\includegraphics[width=\textwidth]{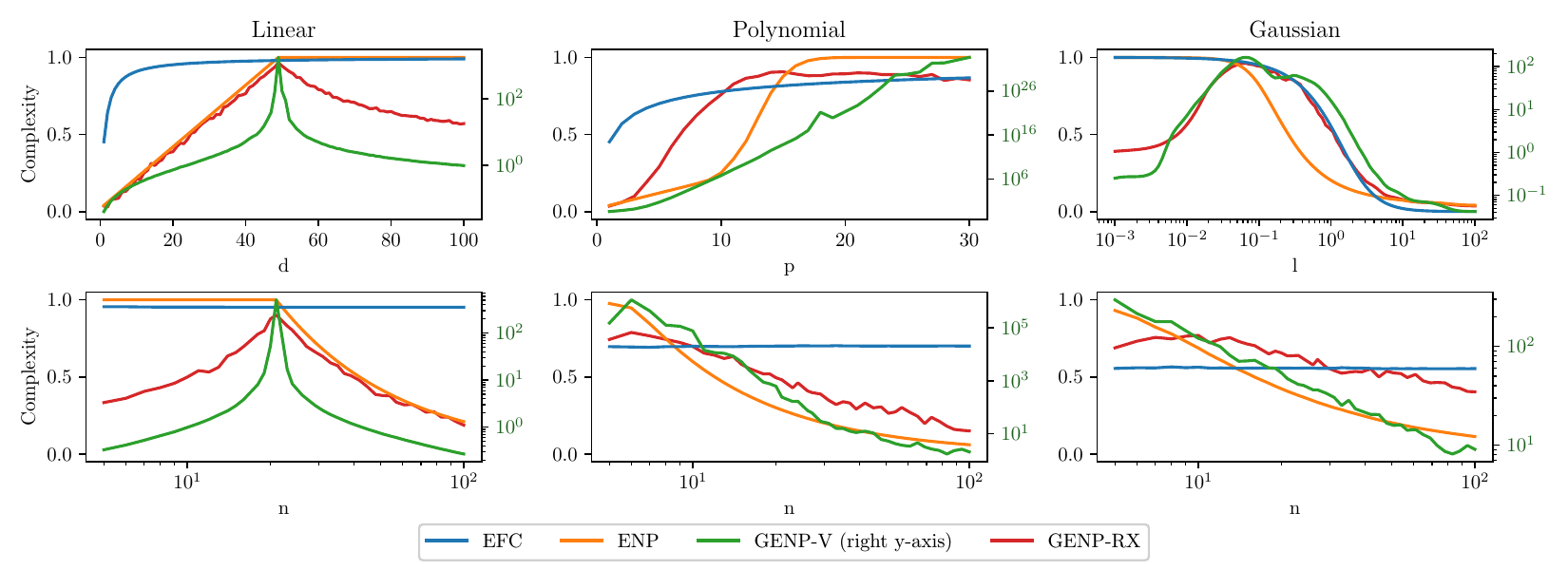}
\caption{Different complexity measures for linear, polynomial, and Gaussian kernel regression. In the top row, we used $n=50$. For the left bottom panel, we used $d=20$; for the remaining four panels, we used $d=1$. For the last two panels in the bottom row, we used $p=5$ and $l=1$. The ENP, GENP-V, and GENP-RX are divided by $n$ to allow direct comparison with the EFC; the GENP-V uses the right y-axis. In contrast to the GENP-V and GENP-RX, the EFC and ENP change monotonically in $d$, $p$, and $l$. The EFC is the only complexity measure that is independent of the sample size $n$. In contrast to the other methods, GENP-V obtains values larger than 1. More details, prediction intervals, and results for additional complexity measures are provided in Appendix~\ref{sec:exp_dets}.}
\label{fig:compl_demo}
\end{figure*}

\section{Analysis of the Feature-Similarity Complexity for Constant-Feature Models}
In this section, we analyze the FC for constant-feature models through the corresponding constant kernel. We study the complexities of the polynomial and Matérn kernels (where the former includes the linear and the latter the Gaussian and Laplace kernels as special cases), as well as those of k-nearest neighbors (kNN), decision trees (DT), and random forests (RF).

\subsection{Polynomial and Matérn Kernels}
\label{sec:pol_mat}
We denote the polynomial kernel as $k_P(\xv,\xpv,p,c):=(c+\xv^\top\xpv)^p$, where $c>0$ and $p\in\mathbb{N}$, and the Matérn kernel, realized as the Matérn covariance function with length scale $l>0$ and differentiability parameter $\nu>0$,\footnote{$k_M(\xv,\xpv,l,\nu)=\frac{2^{1-\nu}}{\Gamma(\nu)}\cdot\left(\sqrt{2\nu}\cdot\frac{\|\xv-\xpv\|_2}{l}\right)^\nu \cdot K_{\nu}\left(\sqrt{2\nu}\cdot \frac{\|\xv-\xpv\|_2}{l}\right)$, where $\Gamma(\nu)$ denotes the gamma function and $K_\nu$ denotes the modified Bessel function of the second kind.} as $k_M(\xv,\xpv,l,\nu)$.
In Theorem~\ref{thm:fc_pm_kern_fct}, we show that the FC increases with the polynomial degree of a polynomial kernel, and decreases with the kernel length scale of a Matérn kernel, and thus naturally generalizes the polynomial degree and length scale as complexity measures. 

\begin{thm}
\label{thm:fc_pm_kern_fct}~\\
For any distribution of the input data with $\Pr(\xv=\xpv)<1$, for fixed $c$ and $\nu$,
\begin{enumerate}[label=(\alph*)]
\item\label{thm:fc_pm_kern-pp}
$\FC(k_P)$ increases strictly in $p$, from $0$ for $p=0$ to $1-\Pr(\xv=\xpv)$ for $p=\infty$.
\item\label{thm:fc_pm_kern-ml}
$\FC(k_M)$ decreases strictly in $l$, from $1-\Pr(\xv=\xpv)$ for $l=0$ to $0$ for $l=\infty$.
\end{enumerate}
\end{thm}
\begin{remark}
The assumption $\Pr(\xv=\xpv)<1$ is extremely weak, since the opposite, $\xv\stackrel{\text{a.s.}}=\xpv$, would mean that all observations are identical. In this case, $\E(\kb(\xv,\xpv)^2)=\kb(\xv,\xv)^2=1$ and the complexities are 0 regardless of the kernel parameters.
\end{remark}

In Theorem~\ref{thm:fc_pm_kern_data}, we analyze how the complexity depends on the dimensionality $d$ and variance $\sigma^2$ of the input data. 
The FC increases with both the dimensionality and the variance, i.e., with the complexity of the input data.

\begin{thm}
\label{thm:fc_pm_kern_data}~\\
If $\xv,\xpv\sim\N(\nv,\sigma^2\cdot\I_d)$ (isotropic normal distribution in $\R^d$) are independent, then, for fixed $p$, $c$, $l$, and $\nu$,
\begin{enumerate}[label=(\alph*),resume]
\item\label{thm:fc_pm_kern-pd} $\FC(k_P)$ increases in $d$, from $0$ for $d=0$ to $1$ for $d=\infty$.
\item\label{thm:fc_pm_kern-ps} $\FC(k_P)$ increases in $\sigma$, from $0$ for $\sigma=0$ to $1-\prod_{i=0}^{p-1}\frac{1+2i}{d+2i}$ for $\sigma=\infty$.
\item\label{thm:fc_pm_kern-md} $\FC(k_M)$ increases strictly in $d$, from $0$ for $d=0$ to $1$ for $d=\infty$.
\item\label{thm:fc_pm_kern-ms} $\FC(k_M)$ increases strictly in $\sigma$, from $0$ for $\sigma=0$ to $1$ for $\sigma=\infty$.
\end{enumerate}
\end{thm}

\begin{figure*}[t]
\center
\includegraphics[width=\textwidth]{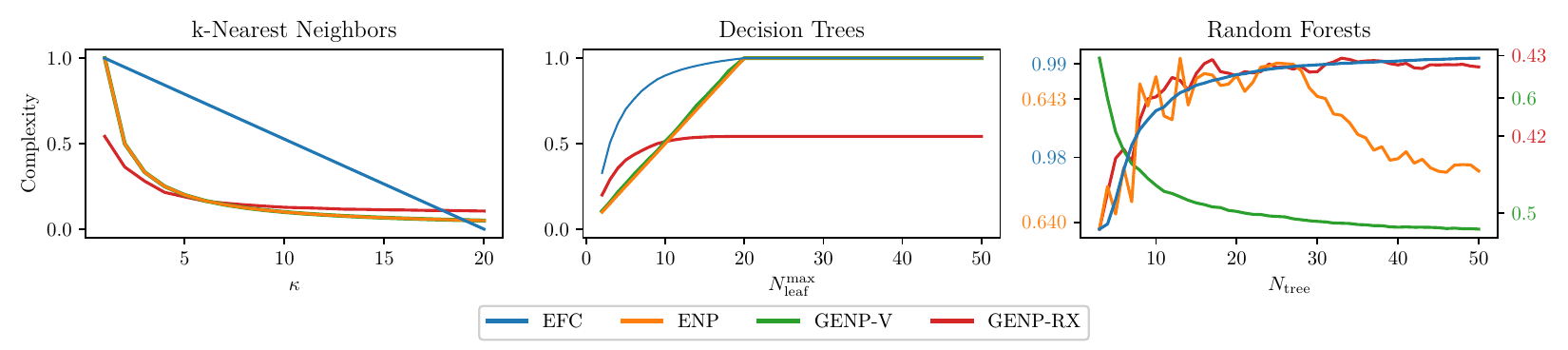}
\caption{Different complexity measures for k-nearest neighbors, decision trees, and random forests, for $n=20$ and $d=1$. The ENP, GENP-V, and GENP-RX are divided by $n$ to enable direct comparison with the EFC. In the third panel, each graph uses its own color-coded y-axis. All complexity measures change monotonically in $\kappa$ and $N^{\text{max}}_{\text{leaf}}$, while this is the case only for the EFC and GENP-V in $N_{\text{tree}}$. The GENP-V, however, decreases as the number of trees increases. More details, prediction intervals, and results for additional complexity measures are provided in Appendix~\ref{sec:exp_dets}.}
\label{fig:compl_demo_emp}
\end{figure*}

In Figure~\ref{fig:compl_demo}, we compare the EFC to the ENP, and its out-of-sample generalizations proposed by \citet{curth2023u} (GENP-V) and \citet{patil2024revisiting} (GENP-RX) for linear regression across data dimensionality $d$, polynomial regression across polynomial degree $p$, and Gaussian kernel regression across kernel length scale $l$. We also compare the complexities for different training data sample sizes $n$. As proposed by Theorems~\ref{thm:fc_pm_kern_fct} and \ref{thm:fc_pm_kern_data}, and Remark~\ref{rmk:efc_n}, the EFC increases with $d$ and $p$, decreases with $l$, and is independent of $n$. The ENP also behaves as expected in terms of $d$, $p$, and $l$---although it significantly changes its behavior at $n=d$, where the model starts to interpolate the data---but is highly dependent on $n$. The GENP-V and GENP-RX tend not to be monotone in any of the parameters. For linear regression, they obtain maxima at $n=d$, which is not unexpected since they are implicitly basically defined in terms of the generalization error\footnote{See Appendix~\ref{sec:genp} for details.}, which is known to diverge at $n=d$ for ridgeless linear regression \citep{hastie2022surprises}.


\subsection{k-Nearest Neighbors, Decision Trees, and Random Forest}
\label{sec:knndtrf}
Although kNN, DT, and RF are non-parametric models, they can all be expressed as kernel smoothers.\footnote{For a kernel $k(\xv,\xpv)$, the corresponding kernel smoother is given by $\fh(\xsv)= \frac{\sum_{i=1}^n k(\xsv,\xvi)y_i}{\sum_{i=1}^nk(\xsv,\xvi)}$, where $\{(\xvi,y_i)\}_{i=1}^n$ denotes the training data; see Appendix~\ref{sec:kern_smooth} for more details.}
The kNN kernel $k_{kNN}$ is 1 if $\xpv$ is in the $k$-neighborhood of $\xv$ and zero otherwise, while the decision tree kernel $k_{DT}$ is 1 if $\xv$ and $\xpv$ are in the same leaf and zero otherwise; see Appendix~\ref{sec:kern_smooth} for details. These kernels are highly data-dependent, lack closed-form expressions, and do not necessarily have corresponding feature expansions. Nonetheless, they can still be used to compute the EFC through its kernel formulation. 
Note, however, that since the kernels are now data-dependent, the EFC is no longer independent of $n$.

\begin{prop}~\\
\label{thm:knn}
For training data $\X\in\R^{n\times d}$,
\begin{enumerate}[label=(\alph*)]
\item
$\EFC(k_{kNN},\X)=1-\frac{\kappa-1}{n-1}$, where $\kappa\leq n$ denotes the number of neighbors considered and $n$ denotes the number of training observations. Thus $\EFC(k_{kNN},\X)$ decreases strictly with $\kappa$ from 1 for $\kappa=1$ to 0 for $\kappa=n$. 
\item
$\EFC(k_{DT},\X)$ increases strictly when a new split is added to the tree, from $0$ for zero splits (and one region) to $1$ when each training point has its own region.
\end{enumerate}
\end{prop}
\begin{remark}
For kNN, the non-normalized ENP is given by $n/\kappa$, and the normalized ENP by $1/\kappa$. Just as for the EFC, it is 1 for $\kappa=1$, but for $\kappa=n$ it is $1/n$, only obtaining a complexity of 0 for an infinite sample size. The EFC, in contrast, is zero for $\kappa=n$---regardless of $n$---in which case the model always predicts the mean of the training data.
\end{remark}

The random forest is an ensemble method in which $N_{\text{tree}}$ decision trees are constructed on different bootstrap samples of the training data, and the final model prediction is the average (for regression) or the majority vote (for classification) of the $N_{\text{tree}}$ trees. When the correlation between the trees is low, the predictive performance of the forest tends to be much better than that of a single tree.
In Theorem~\ref{thm:fc_ens}, we investigate the complexity of the average of $B$ models with i.d. (identically distributed, but not necessarily independent) training datasets.

\begin{thm}
\label{thm:fc_ens}
Let $\{\X_b\}_{b=1}^B$ denote $B$ identically distributed (but not necessarily independent) training datasets, and let $k_b$ denote the kernel obtained from $\X_b$. Let $\rho_{bb'}$ denote the correlation between the normalized kernels $\kb_b$ and $\kb_{b'}$. Let $k_B:=\frac1B\cdot\sum_{b=1}^Bk_b$ denote the kernel of the ensemble model. If each $k_b$ has constant diagonal, i.e., $k_b(\xv,\xv)=k_b(\xpv,\xpv)$, then
\begin{equation*}
\begin{aligned}
\label{eq:fc_ens}
\FC\left(k_B\right)=
\FC(k_b)+\left(1-\frac1B\right)\cdot\left(1-\rho_{bb'}\right)\cdot\Var(\kb_b).
\end{aligned}
\end{equation*}
\end{thm}
\begin{remark}
When the $k_b$s do not have constant diagonals, Theorem~\ref{thm:fc_ens} becomes slightly more complicated, but its core remains. See Appendix~\ref{sec:kern_smooth} for details.
\end{remark}

According to Theorem~\ref{thm:fc_ens}, the complexity of the aggregated model is the complexity of a single model plus a term that depends on the number of models, the correlation between the models, and the variance of the kernel-induced correlations between observations. The aggregated complexity increases with the number of models and decreases with their correlation. It also increases with the variance of the kernel-induced correlations---the lower this variance is, the closer the kernel is to being constant, implying that all models are identical.

In Figure~\ref{fig:compl_demo_emp}, we compare the EFC to the ENP, GENP-V, and GENP-RX for kNN across the number of neighbors $\kappa$, DT across the maximum number of leaves $N^{\text{max}}_{\text{leaf}}$, and RF across the number of trees $N_{\text{tree}}$. Note that since the kernels are now data-dependent, the models are not automatically linear smoothers,\footnote{i.e., $\fhv=\Ss\yv$, $\fhsv=\Sss\yv$, where $\Ss$ and $\Sss$ are independent of $\yv$.} which means that the ENP and GENP-V, which rely on this assumption, should be thought of as approximations.
For kNN, all four measures of complexity decrease with $\kappa$; however, the EFC is the only method that obtains zero complexity for $\kappa=n$. For DT, all four measures increase with $N^{\text{max}}_{\text{leaf}}$, although the ENP and GENP-V significantly change their behavior at $N^{\text{max}}_{\text{leaf}}=n$. For RF, the EFC and GENP-RX increase with the number of trees, while the GENP-V decreases; the ENP first increases and then decreases.

\section{Calculating the Feature-Similarity Complexity in Practice}
Since the true distribution of the inputs $\xv$ is usually unknown, in practice, we calculate the EFC rather than the FC. Since this is an unbiased estimate, as long as the sample is sufficiently large to be representative, it suffices to evaluate the EFC on a subsample of the data, thereby substantially reducing the computational cost. Note, however, that for the kernels in Section~\ref{sec:knndtrf}, the EFC depends on the sample size.

For models where the features change during training, such as neural networks and gradient boosting, it would be misleading to represent the total complexity by the complexity of a single feature configuration. The total model is constructed iteratively and thus depends on all feature configurations during training. We must thus aggregate the complexities of all these features, where features that contribute more to the final model should contribute more to the total complexity.\footnote{Consider the case where a very simple model is used for training until the training error is very close to zero, and then a very complex model is used during the last phase, from a training error very close to zero all the way to zero. The complexity of the total model is, of course, much closer to that of the first, simple model (although slightly larger, since the second, complex model adds a small contribution). Representing the total complexity as the last, or maximum, value during training would result in a very complex model, completely ignoring that virtually all of the model's contribution comes from the simple model. Using a weighted average that accounts for the entire training process is much more representative.} For $T$ training steps, we define the total EFC as the weighted average of the EFCs of the $T$, potentially different, feature configurations during the training process, where the weights are the differences in the loss function: 
\begin{equation}
\label{eq:tot_fc}
\begin{aligned}
&{\EFC}(\{\phiv_t\}_{t=1}^T,\{\xvi\}_i):= \frac{\sum_{t=1}^T\EFC(\phiv_t,\{\xvi\}_i)\cdot\Delta L_t}{\sum_{t=1}^T\Delta L_t},
\end{aligned}
\end{equation}
where $\phiv_t$ is the feature configuration at training step $t$ and $\Delta L_t:=\max(L_{t-1}-L_t,0)$ is the decrease in the loss function between training step $t-1$ and $t$.
Through this definition, iterations for which the updates---and thus $\Delta L$---are large contribute more to the total complexity, while instances with smaller updates---and smaller $\Delta L$---contribute less to the total. Just as the EMC \citep{nakkiran2020deep}, the aggregated EFC thus depends on the entire training procedure.

Typically, the complexity increases during training (see Appendix~\ref{sec:comp_time}), while $\Delta L$ decreases, meaning that high complexities near convergence contribute less to the total. Note that when the features are constant during training, i.e., $\phiv_t=\phiv$ for all $t$, we have 
$\frac{\sum_{t=1}^T\EFC(\phiv,\{\xvi\}_i)\cdot\Delta L_t}{\sum_{t=1}^T\Delta L_t}=\EFC(\phiv,\{\xvi\}_i)$.

\begin{figure*}[t!]
\center
\includegraphics[width=0.85\textwidth]{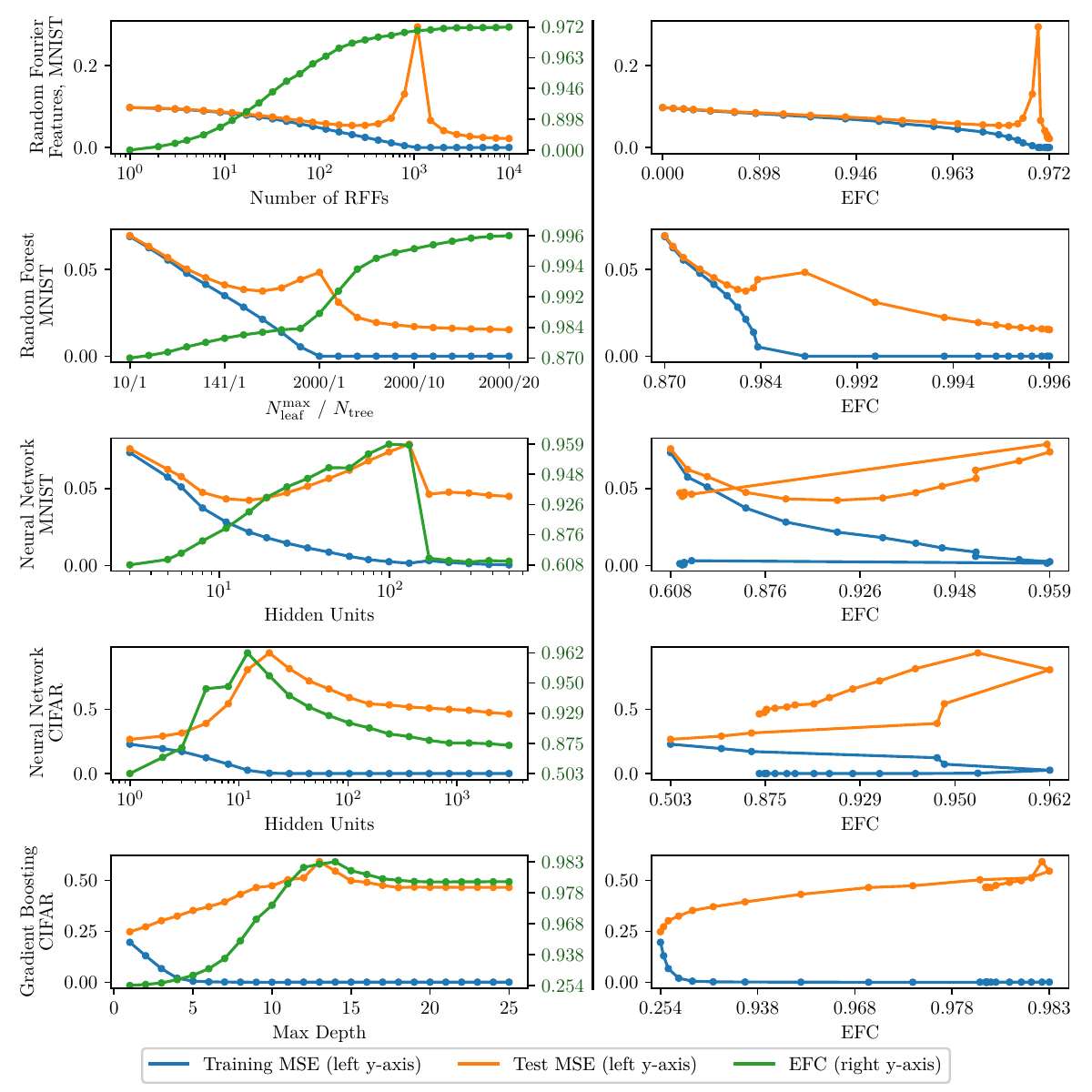}
\caption{Revisiting the double descent experiments of \citet{belkin2019reconciling} for four different models on two different datasets. In the left column, we plot the training and test MSEs (left y-axis) and the EFC (right y-axis) as functions of the originally proposed complexity measure, while in the right column, we plot the training and test MSEs as functions of the EFC. When measuring the model complexity in terms of the EFC, there is still a double descent behavior for the constant-feature models, but not for the models where the features change during training. Details and prediction intervals are provided in Appendix~\ref{sec:exp_dets}.}
\label{fig:dd}
\end{figure*}

\section{Revisiting Double Descent Based on the Feature-Similarity Complexity}
\label{sec:dd}
Double descent, the phenomenon that models eventually generalize better with increasing complexity, was popularized by \citet{belkin2019reconciling}, who demonstrated it across several methods and datasets. However, for some of their models, the definitions of complexity are somewhat arbitrary, which raises questions about their conclusions. \citet{curth2023u} even argued that double descent is an artifact; using the GENP-V as a complexity measure, they demonstrated that the double descent phenomenon disappears. However, as discussed earlier, the GENP-V (and GENP-RX) implicitly define function complexity through the generalization error, which, as demonstrated in Figure~\ref{fig:compl_demo}, is not necessarily the same thing as model complexity. Thus, by construction, there is never a double descent for these complexity measures.

In Figure~\ref{fig:dd}, we revisit the experiments of \citet{belkin2019reconciling},\footnote{We repeat the three experiments in the main part of the paper, which are the three experiments conducted on the MNIST dataset \citep{lecun1998gradient} in our figure. We also repeat the neural network experiment on the CIFAR dataset \citep{krizhevsky2009learning} from Appendix C in the paper, and add a comparison to gradient boosting with decision trees on CIFAR.} and in addition to the original complexity metrics, we also use the EFC.

For random Fourier features, \citet{belkin2019reconciling} use the number of features as a proxy for complexity. According to our results, the EFC increases with the number of features, and we also observe a second descent for the EFC.

For random forests, the complexity definition used by \citet{belkin2019reconciling} is rather intricate. They first use a forest consisting of a single tree and increase the maximum number of leaves from 1 to 2000. Then, they keep the maximum number of leaves constant at 2000 and increase the number of trees in the forest from 1 to 20. Although it is not obvious that this is the best way to control the complexity of a random forest, their conclusion is supported by the EFC, which increases with their definition of complexity, and we again see a second descent also for the EFC.

For neural networks, \citet{belkin2019reconciling} use the number of hidden nodes as a proxy for complexity. They, however, initialize the parameters differently in the under- and over-parameterized regimes. Training networks of increasing size sequentially, for under-parameterized networks, they use the trained weights of the previous, slightly smaller network to initialize the next. In contrast, in the over-parameterized regime, the weights of the network are randomly initialized. For the MNIST data, the second descent occurs exactly at this transition, and the double descent in this case may well be an artifact of the model change. This is confirmed by the EFC: the initialization scheme strongly affects the model complexity, and the randomly initialized, over-parameterized models are, in fact, less complex in terms of the EFC, where there is no double descent.

For the CIFAR data, in contrast to \citet{belkin2019reconciling}, we used the same random initialization scheme for all models. Despite this, we see a double descent phenomenon\footnote{Since even the simplest models are complex enough to perform well, there is no first descent, only a second.} in the number of hidden units. This phenomenon, however, disappears when complexity is measured in terms of the EFC. We observe the same phenomenon with gradient boosting on decision trees, where we use the maximum tree depth as a proxy for complexity.

Our results thus suggest that double descent indeed occurs for constant-feature models such as random Fourier features and random forests (and also for linear regression; \citet{hastie2022surprises} showed a double descent phenomenon in the number of parameters $d$ in this case). However, for models in which the features change during training, such as neural networks and gradient boosting, the phenomenon is much less obvious: in the examples above, the complexity eventually decreases with model capacity (number of parameters/tree depth). Apparently, high-capacity models can more easily adapt to the data, without using a more complex model than necessary.

\section{Discussion, Limitations, and Conclusions}
We introduced the Feature-Similarity Complexity (FC), a model complexity measure based on the similarities of the model's features across inputs. We showed how, for constant-feature models, the FC is independent of the model parameters, and thus measures \emph{model complexity} rather than \emph{function complexity}---even when calculated from a single learned function. We additionally proved that the FC generalizes the polynomial degree, kernel length scale, data dimensionality, number of neighbors, number of leaves, and number of trees as complexity measures. Due to its generality---it is well-defined for any parametric or kernel-based non-parametric regression or classification method---and its ease of calculation, we believe that the FC will be useful for better understanding problems related to model complexity. As an illustration, we used it to obtain new insights into the double descent phenomenon.

\textbf{Limitations:} The FC does not capture hard parameter constraints such as $\|\thetahv\|\leq M$. Such constraints can, however, usually be transferred from the parameters to the features. As an example, $\fh(\xv,\thetahv)=\thetahv^\top\phiv(\xv)$, $\|\thetahv\|\leq M$ can equivalently be expressed as $\fh(\xv,\thetatv)=\frac{M}{1+\|\thetatv\|}\cdot\thetatv^\top\phiv(\xv)$ with no constraints on $\thetatv$.\footnote{\label{fn:mv_const} See Appendix~\ref{sec:proofs} for the proof.}\\
Whenever the features are expensive to compute, e.g., for deep neural networks, so is the FC, although evaluating it on a subsample can significantly reduce computational cost. However, some non-parametric models require evaluations on the entire training data.

\textbf{Future work:} One interesting future direction of research would be to monitor how the complexity of a model changes during training to better understand phenomena such as grokking. Another interesting direction is to replace the expected value in the FC's definition with a more fine-grained representation of the distribution.

\section*{AI Use Statement}
In this work, we used generative AI tools (ChatGPT and Claude) as a discussion and brainstorming partner on different topics, including understanding and formulating theoretical and mathematical concepts; proving theorems; and pedagogically expressing ideas and concepts. We also used it to understand relevant references more quickly and effectively, and to generate and refactor code. We have not used AI to find or write references; generate, edit, or otherwise manipulate/handle data; or generate/translate text (except to check spelling, grammar, and phrasing with Grammarly). 

We have reviewed and internalized all AI-assisted work. All text, including proofs, is our own writing, although sometimes based on ideas from, and discussions with, AI. We have carefully reviewed and tested all AI (and human) generated code. We take full responsibility for the entire content of this work.

\section*{Acknowledgements}
This research was supported by the \emph{Wallenberg AI, Autonomous Systems and Software Program (WASP)} funded by Knut and Alice Wallenberg Foundation, by the \emph{Kjell och M{\"a}rta Beijer Foundation}, and by the project \emph{Blending probabilistic and nonlinear representations} (contract number: 2025-04318), funded by the Swedish Research Council.

\newpage
\bibliography{refs}
\bibliographystyle{apalike}
\newpage

\appendix
\newpage

\section{More on the Difference Between Function and Model Complexity}
\label{sec:fct_mod}

To illustrate the difference between function and model complexity, we use Gaussian kernel ridge regression as an example. For the Gaussian kernel $k_G(\xv,\xpv,l)=\exp\left(-\frac{\|\xv-\xpv\|_2^2}{2l^2}\right)$, training data $(\X_t,\yv)\in\R^{n_t \times d}\times \R^{n_t}$, new observation $\xsv\in\R^d$, kernel matrices $\K_{tt}\in\R^{n_t\times n_t}$ and $\bm{k}_t\in\R^{n_t}$, where $(\K_{tt})_{ij}=k_G(\xvi,\xvj,l)$, where  $\xvi^\top$ and $\xvj^\top$ are rows in $\X_t$, and $(\bm{k}_t)_i=k_G(\xsv,\xvi,l)$, regularization $\lambda\geq 0$, and identity matrix $\I_{n_t}\in\R^{n_t\times n_t}$,
\begin{equation}
\label{eq:krr_gauss}
\fh(\xsv)=\bm{k}_t(\xsv,l)^\top(\K_{tt}(l)+\lambda\cdot\I_{n_t})^{-1}\yv,
\end{equation}
which, using the identity $k(\xv,\xpv)=\phiv(\xv)^\top \phiv(\xpv)$, equivalently can be expressed as linear ridge regression in feature space (see e.g. \citet{welling2011first} for details).

While the kernel length scale $l$ enters Equation~\ref{eq:krr_gauss} through the kernel and thus affects both the features (remember that $k(\xv,\xpv)=\phiv(\xv)^\top \phiv(\xpv)$) and the inferred function, the regularization $\lambda$ does not affect the kernel/features---only the inferred functions. Thus, changing the kernel through $l$ affects both the model and function complexities, while changing the regularization $\lambda$ only affects the function complexity. Or, put the other way, the model complexity is a function only of $l$, while the function complexity is a function of both $l$ and $\lambda$.

In Figure~\ref{fig:pen_demo}, we illustrate this difference by comparing the EFC (which measures model complexity) and the ENP (which measures function complexity) for two functions with different length scale and regularization.
The green function uses a shorter length scale and thus a more complex model than the yellow one. It is, however, more heavily regularized and thus simpler in terms of the ENP.\footnote{\label{fn:enp_mse} The ENP essentially measures the mean squared error (MSE) of the inferred function: For a linear smoother, where the predictions on training data are given by $\fhv=\Ss\yv$, $\text{ENP}:=\Tr(\Ss)$. When $\Ss$ is symmetric with eigenvalues in $[0,1]$, which is the case for kernel regression, the MSE is upper bounded by $(n-\text{ENP})^2\cdot\frac1n\cdot\|\yv\|_2^2$ (see Appendix~\ref{sec:proofs} for the proof).} 
\begin{figure}[t]
\center
\includegraphics[width=0.5\textwidth]{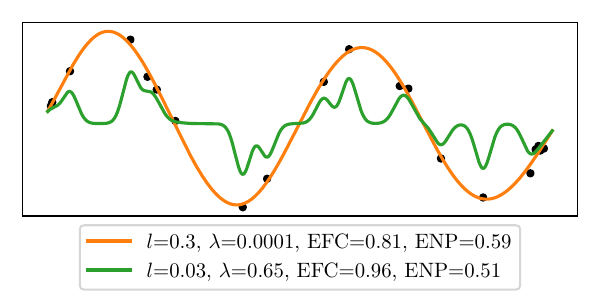}
\caption{Inferred functions of Gaussian kernel ridge regression for two different length scales with different regularizations. The green function has a shorter length scale and thus a larger model complexity, which is captured by the EFC. However, since it uses a higher regularization, it is less complex in terms of the ENP. Since regularization affects only the parameters, not the features, the EFC, which measures model complexity, correctly identifies the more complex model, whereas the ENP, which measures function complexity, does not.}
\label{fig:pen_demo}
\end{figure}
%
%

\subsection{On the Function and Model Complexities of Gaussian Processes}
If we assume that $\thetahv$ is a random vector with $\Cov(\thetahv)=\I_p$, where $\I_p\in\R^{p\times p}$ denotes the identity matrix, as is done for Gaussian processes, then $\text{corr}_{\thetahv}(\fh(\xv),\fh(\xpv))=\kb(\xv,\xpv)$.\footnote{\label{fn:corr_k}See Appendix~\ref{sec:proofs} for the proof.} In this case, both the function complexity (measured as the correlation between functions) and model complexity (measured as the correlation between features) depend directly on $k$ and thus coincide.\footnote{Equivalently, the Gaussian process kernel and the neural tangent kernel coincide.} Note, however, that in Gaussian process regression (Kriging), this applies only to the prior functions; for the posterior functions, the distribution of $\thetahv$ depends on the data, and we cannot assume that $\Cov(\thetahv)=\I_p$. Thus, the model and function complexities may differ substantially for the posterior functions.

We first give a brief review of Gaussian processes and Gaussian process regression (Kriging).

The $n$ predictions of a Gaussian process, GP, are given by
$$\fhv\sim\N(\nv,\K),$$
where $\fhv\in\R^n$ denotes the vector of predictions, $\K\in\R^{n\times n}$ denotes the covariance (kernel) matrix of the Gaussian process, and $\mathcal{N}(\cdot,\cdot)$ denotes the multivariate normal distribution. For simplicity, we assume that the GP has zero mean, but it is possible to replace $\nv$ with any mean vector $\muv\in\R^n$. In the proof of Proposition~\ref{thm:gp}, we use this more general form.

In Gaussian process regression, or Kriging, $\N(\nv,\K)$ is used as a prior for the predictions, while the posterior predictions also depend on the training data $(\X_t,\yv)\in\R^{n_t \times d}\times \R^{n_t}$. The posterior predictions are given by
$$\fhv_{\po}\sim\N\left(\K_t^\top\K_{tt\lambda}^{-1}\yv,\ \K-\K_t^\top\K_{tt\lambda}^{-1}\K_t\right),$$
where $\K_t\in\R^{n_t\times n}$ denotes the kernel matrix evaluated on training and prediction data, $\K_{tt\lambda}=\K_{tt}+\lambda \I_{n_t}$, where $\K_{tt}\in\R^{n_t\times n_t}$ denotes the kernel matrix evaluated on training data, and $\lambda\geq 0$ denotes the regularization, or estimated sample noise.

According to Mercer's theorem, we can decompose $\K_t$ and $\K_{tt}$ as $\K_t=\Ph_t\Ph^\top$ and $\K_{tt}=\Ph_t\Ph_t^\top$, where $\Ph_t\in\R^{n_t\times p}$ and $\Ph\in\R^{n\times p}$ are the feature matrices evaluated on training and prediction data, respectively. For a single observation $\xv$, we denote the corresponding element in $\fhv$ by $\fh(\xv)$, and the corresponding row in $\Ph$ by $\phiv(\xv)^\top$.

In Proposition~\ref{thm:gp}, we show how a Gaussian process can equivalently be formulated as regression in the feature space imposed by the prior kernel. For the prior (data-independent) predictions, the regression parameters are simply i.i.d. standard normal random variables, while the parameters of the data-dependent posterior also depend on the training data and the regularization.
\begin{prop}
\label{thm:gp}
\begin{equation*}
\begin{aligned}
&\fhv\sim\N(\nv,\K)&& \iff \fh(\xv)=\zv^\top\phiv(\xv)\\
&\fhv_{\po}\sim\N\left(\K_t^\top\K_{tt\lambda}^{-1}\yv,\ \K-\K_t^\top\K_{tt\lambda}^{-1}\K_t\right)&& \iff \fh_{\po}(\xv) =\thetahv^\top\phiv(\xv),
\end{aligned}
\end{equation*}
where 
\begin{equation*}
\begin{aligned}
\zv&\sim \N(\nv,\I_p)\\
\thetahv&=\Ph_t^\top\K_{tt\lambda}^{-1}\yv+\sqrt{\I_p-\Ph_t^\top\K_{tt\lambda}^{-1}\Ph_t}\cdot\zv\\
&=(\Ph_t^\top\Ph_t+\lambda\cdot\I_p)^{-1}\Ph_t^\top\yv+\left(\I_p+1/\lambda\cdot\Ph_t^\top\Ph_t\right)^{-1/2}\cdot\zv\\
\end{aligned}
\end{equation*}
\end{prop}
\begin{remark}
 For the linear kernel, where $\Ph_t=\X_t$ and $\phiv(\xv)=\xv$, we obtain $\E(\thetahv)=(\X_t^\top\X_t+\lambda\cdot\I_p)^{-1}\X_t^\top\yv$, i.e., we recover standard linear regression.
\end{remark}
\begin{remark}
From the expression for $\thetahv$, it is easy to see that for $\lambda=0$ (noise-free data) we obtain $\thetahv=\Ph_t^+\yv+\sqrt{\I_p-\Ph_t^+\Ph_t}\cdot\zv$ (where $(\cdot)^+$ denotes the Moore-Penrose pseudo-inverse), which reduces to the deterministic $\thetahv=\Ph_t^+\yv$ for $n_t\geq p$. For $\lambda=\infty$ (infinitely noisy data), we obtain $\thetahv=\zv$, i.e.\ since the data is non-informative, the posterior reduces to the prior.
\end{remark}

The Gaussian process kernels of the prior and posterior functions are given by
\begin{equation*}
\begin{aligned}
&\text{\normalfont GPK}_{\text{\normalfont pr}}=\Cov(\fhv)=\Ph\Cov(\zv)\Ph^\top=\K\\
&\text{\normalfont GPK}_{\po}=\Cov(\fhv_{\po})=\Ph\Cov(\thetahv)\Ph^\top=\K-\K_t^\top\K_{tt\lambda}^{-1}\K_t,
\end{aligned}
\end{equation*}
while the neural tangent kernels of the prior and posterior functions are given by

\begin{equation*}
\begin{aligned}
&\text{\normalfont NTK}_{\text{\normalfont pr}}=\frac{\partial \fhv}{\partial \zv}\cdot\left(\frac{\partial \fhv}{\partial \zv}\right)^\top=\Ph\Ph^\top=\K\\
&\text{\normalfont NTK}_{\text{\normalfont po}}=\frac{\partial \fhv_{\text{\normalfont po}}}{\partial \thetahv}\cdot\left(\frac{\partial \fhv_{\text{\normalfont po}}}{\partial \thetahv}\right)^\top=\Ph\Ph^\top=\K.
\end{aligned}
\end{equation*}
Thus, for the prior (data-ignorant) Gaussian process, the Gaussian process kernel and the neural tangent kernel coincide, and hence the function and model complexities coincide. For the posterior Gaussian process, the Gaussian process kernel, which governs the function complexity, does not coincide with the neural tangent kernel, which governs the model complexity.

In contrast, the prior and posterior NTKs coincide, which is expected. Updating the parameters based on the training data does not affect the features, which govern the model complexity, but it does affect the model predictions, which govern the function complexity.

\section{Additional Experiments and Details}
\label{sec:exp_dets}

In Figure~\ref{fig:compl_demo2}, we extend Figure~\ref{fig:compl_demo} with the von Neumann entropy  (vNE) and parameter norm ($\|\thetahv\|_2^2$) complexity measures. We also add the first and third quartiles over the 100 realizations as dotted lines.
In Figure~\ref{fig:compl_demo_emp2}, we extend Figure~\ref{fig:compl_demo_emp} with the von Neumann entropy  (vNE)---but not the parameter norm, since the methods are non-parametric---and the first and third quartiles over the 100 realizations.
In Figure~\ref{fig:dd2}, we add the first and ninth deciles over the 10 (or 100) realizations of the experiments as dotted lines.
\begin{figure}[t]
\center
\includegraphics[width=\textwidth]{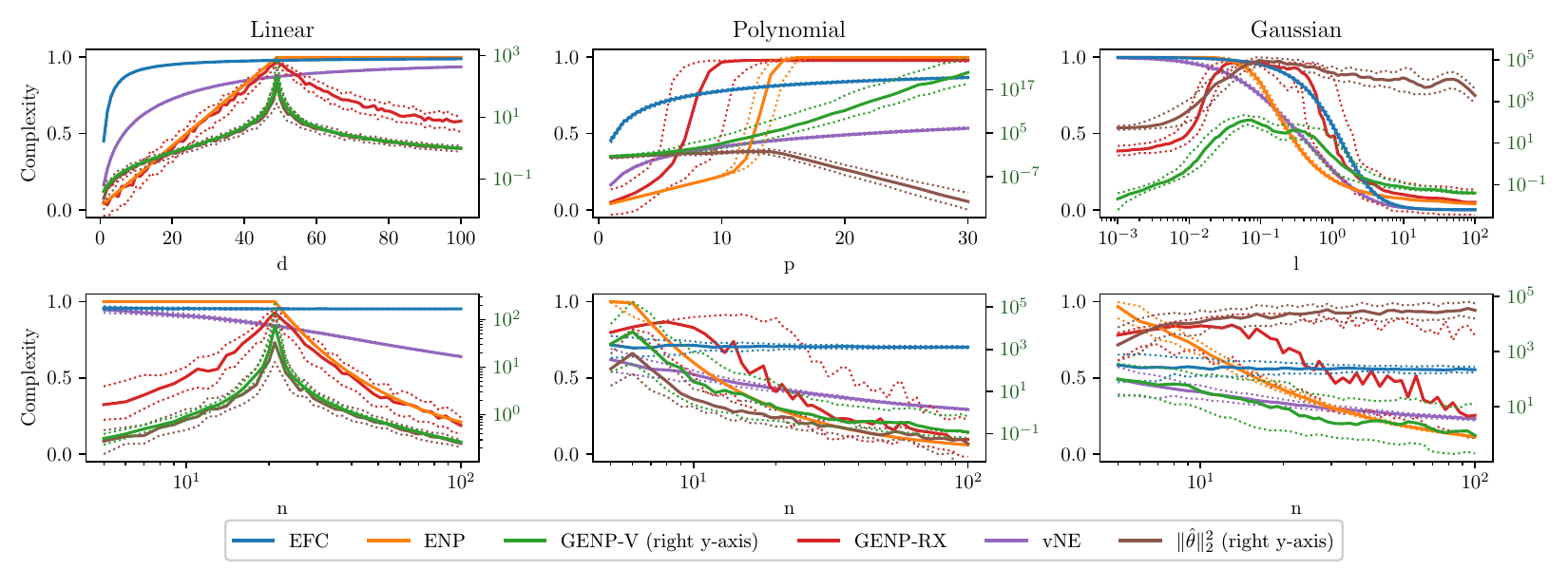}
\caption{The same setting as in Figure~\ref{fig:compl_demo} for two additional complexity measures, including the first and third quartiles over the 100 realizations (dotted lines). The vNE behaves similarly to the EFC, except that it depends on $n$, while $\|\thetahv\|_2^2$ (on the right y-axis) behaves similarly to the GENP-V.}
\label{fig:compl_demo2}
\end{figure}

\begin{figure}[t]
\center
\includegraphics[width=\textwidth]{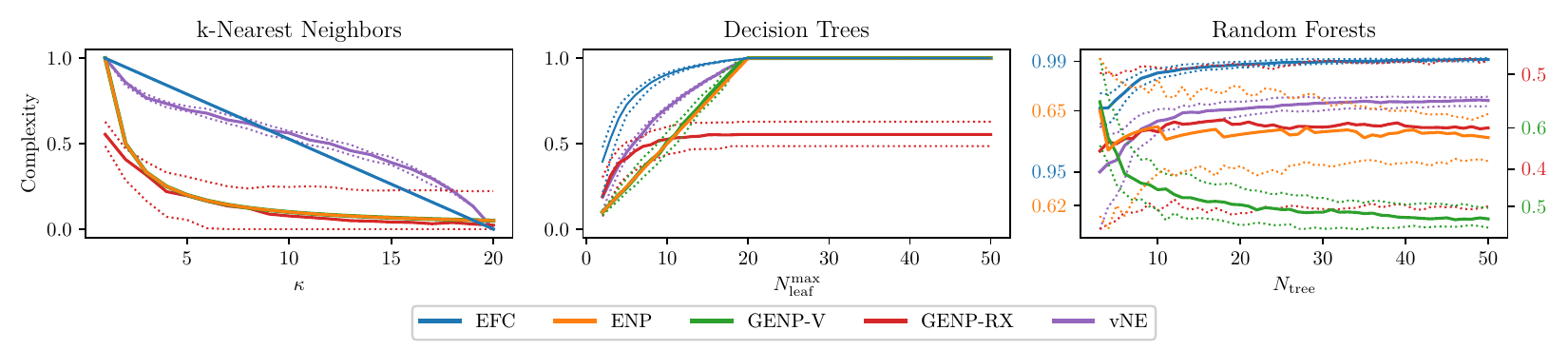}
\caption{The same setting as in Figure~\ref{fig:compl_demo_emp} for vNE, including the first and third quartiles over the 100 realizations (dotted lines). The vNE uses the same y-axis as the EFC and behaves similarly.}
\label{fig:compl_demo_emp2}
\end{figure}

\begin{figure}[t!]
\center
\includegraphics[width=0.9\textwidth]{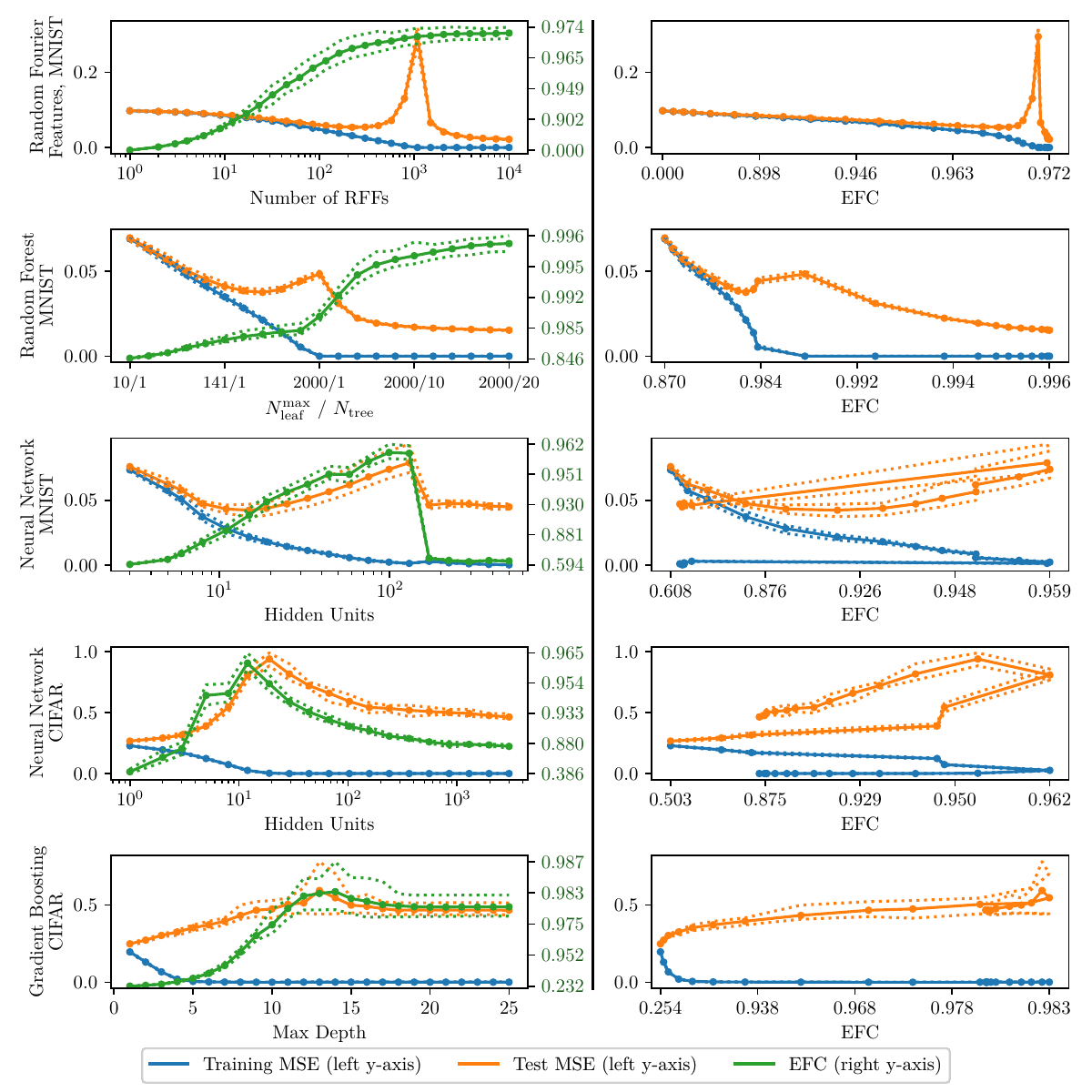}
\caption{The same setting as in Figure~\ref{fig:dd}, including the first and ninth decile over 10 (or 100) realizations of the experiments as dotted lines.}
\label{fig:dd2}
\end{figure}

\subsection{Details for Figures~\ref{fig:compl_demo}, \ref{fig:compl_demo_emp}, \ref{fig:compl_demo2}, and \ref{fig:compl_demo_emp2}}
For training (in-sample) and test (out-of-sample) data, $\X\in\R^{n\times d}$ and $\Xs\in\R^{\ns\times d}$ were generated as $\xv\sim\mathcal{N}(\nv,\I_d)$ and $\xsv\sim\N(\nv,\I_d)$, where $\xv^\top$ and $\xsv{^\top}$ are rows in $\X$ and $\Xs$ respectively, and where we used $\ns=n$.

For Figures~\ref{fig:compl_demo} and~\ref{fig:compl_demo2}, we used the linear, $1+\xv^\top\xpv$; polynomial, $(1+\xv^\top\xpv)^p$; and Gaussian, $\exp\left(-\frac{\|\xv-\xpv\|_2^2}{2l^2}\right)$ kernels respectively. The in- and out-of-sample predictions are given in closed form by $\fhv=\Ss\yv$, $\Ss=\K\left(\K+\lambda\I_n\right)^{-1}$ and $\fhsv=\Sss\yv$, $\Sss=\Ks\left(\K+\lambda\I_n\right)^{-1}$, where $\Ks_{ij}= k(\bm{x_i^*},\bm{x_j})$ and $\K_{ij}= k(\bm{x_i},\bm{x_j})$.
To avoid singular matrices, we used $\lambda=10^{-5}$.

For Figures~\ref{fig:compl_demo_emp} and \ref{fig:compl_demo_emp2}, we used the kernel smoother formulations from Appendix~\ref{sec:kern_smooth}.

In Figures~\ref{fig:compl_demo} and \ref{fig:compl_demo_emp}, the results are reported as the mean across 100 random seeds. In Figures~\ref{fig:compl_demo2} and  \ref{fig:compl_demo_emp2}, they are reported as the median and first and third quartiles over the 100 seeds.

The EFC was evaluated on $\K$. The normalized ENP was calculated as $\frac1n\Tr(\Ss)$. The normalized GENP-V was calculated as $\frac1{\ns}\Tr(\Sss{^\top}\Sss)$. For the normalized GENP-RX, we calculated the optimism as $\frac1{\ns}\|\ysv-\bm{\Sss}\yv\|_2^2-\frac1n\|\yv-\Ss\yv\|_2^2$, inserted it into Equation 22 in \cite{patil2024revisiting} to obtain GENP-RX, and then divided by $n$ to obtain the normalized version. The normalized vNE was evaluated on $\K$. For $\|\thetahv\|_2^2$, we used $\yv\sim\N(\nv,\I_{n})$ and $\bm{\hat{\alpha}}=(\K+\lambda\I_n)^{-1}\yv$ to calculate $\|\thetahv\|_2^2=\|\bm{\hat{\alpha}}\|_{\K}^2=\bm{\hat{\alpha}}^\top\K\bm{\hat{\alpha}}$, where we use the equivalence of kernel ridge regression and ridge regression in the corresponding feature expansion to report the norm of the parameters for the feature expansion.

\subsection{Details for Figure~\ref{fig:dd}}
For all experiments, we used 1000 training and test samples, respectively, except for random forests, where we used 10000 samples for training and testing, respectively. 
All results are reported as the mean across 10 random seeds, except for random Fourier features, where we used 100 seeds. In Figures~\ref{fig:dd2} and \ref{fig:rf_boot}, we additionally report the first and ninth deciles as dotted lines.
For the two neural-network experiments, the EFC was computed every 5th epoch on a random subset of 20 training observations. For the remaining three experiments, we calculated the EFC on the entire training data. For gradient boosting, we calculated the EFC at every iteration.

Just as \citet{belkin2019reconciling}, we used the Gaussian kernel with length scale 5 for the random Fouriel features.

Just as \citet{belkin2019reconciling}, we disabled the bootstrap resampling for random forests. However, using bootstrap resampling does not change anything conceptually; in Figure~\ref{fig:rf_boot} we present the analogous results with bootstrap resampling enabled.

\begin{figure}[t!]
\center
\includegraphics[width=0.9\textwidth]{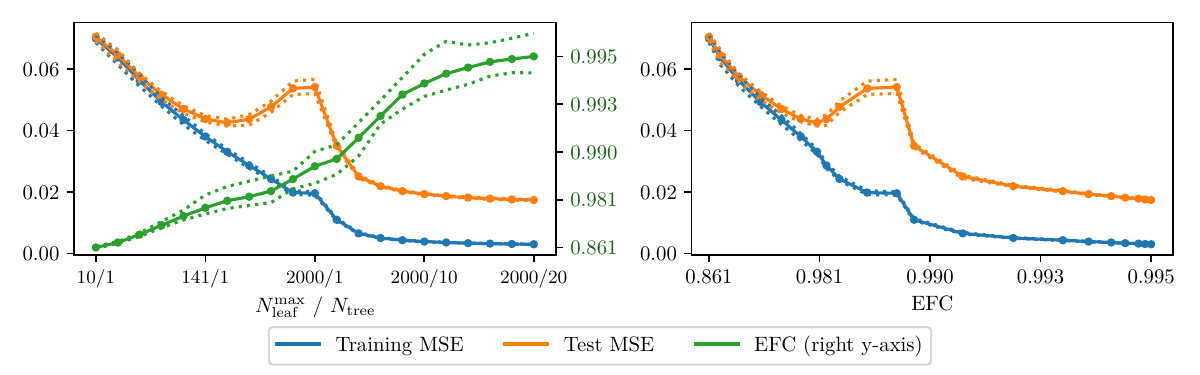}
\caption{Analog of the second row of Figure~\ref{fig:dd} with bootstrap resampling enabled.}
\label{fig:rf_boot}
\end{figure}

For the neural networks on MNIST, the images were downsampled to $8\times 8$ pixels. We used a single hidden-layer neural network with ReLU and linear (i.e., not softmax) activation functions in the hidden and output layers, respectively, a learning rate of 0.01, a momentum of 0.95, and trained with the squared loss for 100000 full-batch epochs. The transition between the under- and over-parameterized regimes occurs for 10000 parameters (10 classes times 1000 training samples).

For the neural networks on CIFAR, as \citet{belkin2019reconciling}, we used only the cat and dog classes, downsampled the images to $8\times 8$ pixels, and converted them to grayscale. We used a single hidden-layer neural network with ReLU and linear (i.e., not sigmoid) activation functions in the hidden and output layers, respectively, a learning rate of 0.0001, a momentum of 0.99, and trained with the squared loss for 10000 epochs using 10 mini-batches. 

For gradient boosting on CIFAR, we used the same data as for neural networks, trained with a learning rate of 0.01, and a momentum of 0.95 for 100 stages.

The MNIST data are available at \url{http://yann.lecun.com/exdb/mnist/} under the CC-BY-SA 3.0 license. The CIFAR-10 data are available at \url{https://www.cs.toronto.edu/~kriz/cifar.html} under the MIT license.

Each experiment took up to 5 minutes (for the neural networks on MNIST) to run on a laptop with an Intel i9-13980HX CPU, an NVIDIA GeForce RTX 4070 GPU, and 32 GB of RAM. The total runtime for all experiments was just below 100 minutes.

We used ChatGPT as a coding assistant.

\section{Connections between the GENP-V, the GENP-RX, and the Generalization Error}
\label{sec:genp}
For the out-of-sample response and predictions $\ysv,\fhsv\in\R^{\ns}$, the generalization error is given by $\frac{1}{\ns}\cdot\|\ysv-\fhsv\|_2^2$.

For a linear smoother, with out-of-sample predictions given by $\fhsv=\Sss\yv$, the GENP-V is defined as $\text{GENP-V}:=\frac{n}{\ns}\cdot\Tr(\Sss{^\top}\Sss)$, which, according to Proposition~\ref{thm:genp-v}, is closely related to the generalization error.

\begin{prop}
\label{thm:genp-v}
Assume that $\fhsv=\Sss\yv$ for $\fhv\in\R^{\ns}$, $\yv\in\R^n$, so that the generalization (or out-of-sample) error is given by $\frac1{\ns}\cdot\|\ysv-\Sss\yv\|_2^2$. If $\yv\sim\N(\nv,\sigma^2_y\cdot\I_n)$ and $\ysv\sim\N(\nv,\sigma^2_y\cdot\I_{\ns})$ are independent, then, 
$$\E_y\left(\frac1{\ns}\cdot\|\ysv-\fhsv\|_2^2\right)=\sigma^2_y\cdot\left(1+\frac1n\cdot\text{\normalfont GENP-V}\right).$$
\end{prop}

The GENP-RX is defined in terms of the optimism, i.e., \ the difference between the in- and out-of-sample errors: $\frac{1}{\ns}\cdot\|\ysv-\fhsv\|_2^2-\frac{1}{n}\cdot\|\yv-\fhv\|_2^2$. Assuming that $\frac{1}{n}\cdot\|\yv-\fhv\|_2^2$ is bounded (in the over-parameterized region, it is 0), the optimism becomes very closely related to the generalization error.

\section{Kernel Smoothers}
\label{sec:kern_smooth}
We formulate the kernel smoothers for regression, in which case the predicted value is a weighted average over the training data. For classification, the average is simply replaced by the mode (i.e.\ the most common value).
The predictions of a kernel smoother are given by
\begin{equation}
\label{eq:kernel_smoother}
\fh(\xsv)=\frac{\sum_{i=1}^n k(\xsv,\xvi)y_i}{\sum_{i=1}^nk(\xsv,\xvi)}=\frac{\bm{k}(\xsv,\X)^\top\yv}{\bm{k}(\xsv,\X)^\top\bm{1}}=\bm{s}(\xsv,\X)^\top\yv,
\end{equation}
where $\bm{k}(\xsv,\X)\in\R^n$ is a vector such that $\bm{k}(\xsv,\X)_i=k(\xsv,\xvi)$, $\bm{1}\in\R^n$ is a vector of ones, and $\bm{s}(\xsv,\X):=\frac{\bm{k}(\xsv,\X)}{\bm{k}(\xsv,\X)^\top\bm{1}}$.

\subsection{k-Nearest Neighbors}
For k-nearest neighbors regression, each prediction is defined as the average of the $\kappa$ closest training points:
\begin{equation}
\label{eq:knn}
\fh(\xsv)=\frac1\kappa \cdot \sum_{i:\xvi\in N_{\kappa}(\xsv)}y_i,
\end{equation}
where $N_{\kappa}(\xsv)$ is the neighborhood of $\xsv$, consisting of the $\kappa$ closest training points $\{\xvi\}_{i=1}^\kappa$.

With the nearest neighbor kernel, 
$$k_{kNN}(\xsv,\xvi):=
\begin{cases}
1\text{ if }\xvi\in N_\kappa(\xsv)\\
0\text{ if }\xvi\notin N_\kappa(\xsv),
\end{cases}
$$
Equation~\ref{eq:kernel_smoother} becomes exactly Equation~\ref{eq:knn}.

\subsection{Decision Trees}
For decision tree regression, the input space is split into $M$ non-overlapping regions, where each prediction is defined as the average of the training points that belong to the same region as the point of interest:
\begin{equation}
\label{eq:tree}
\fh(\xsv)=\frac1{|R(\xsv)|} \cdot \sum_{i:\xvi\in R(\xsv)}y_i,
\end{equation}
where $R(\xsv)$ is the region that $\xsv$ belongs to, and $|R(\xsv)|$ is the number of training points $\xvi$ in $R(\xsv)$.
With the decision tree kernel, 
$$k_{DT}(\xsv,\xvi):=
\begin{cases}
1\text{ if }\xvi\in R(\xsv)\\
0\text{ if }\xvi\notin R(\xsv),
\end{cases}
$$
Equation~\ref{eq:kernel_smoother} becomes exactly Equation~\ref{eq:tree}.

\subsection{Random Forests}
With $B$ bootstraps of the data, we obtain $B$ decision trees, each with kernel $k_{DT,b}(\xsv,\xvi):= \#(\xvi\in R_b(\xsv))$ for $b=1,\dots B$, i.e., the kernel counts the number of times $\xvi$ occurs in $R_b(\xsv)$. When there is no bootstrap, each $\xvi$ may not occur more than once in $R(\xsv)$, but due to the bootstrap, $\xvi$ may be represented multiple times in the bootstrapped training data if multiple bootstrap samples correspond to the same $\xvi$. The predictions of the random forest are given by the average of the individual trees:
\begin{equation*}
\label{eq:ref}
\begin{aligned}
\fh(\xsv)&=\frac1B\cdot\sum_{b=1}^B\left(\frac1{|R_b(\xsv)|} \cdot \sum_{i:\xvi\in R_b(\xsv)}y_i\right)
=\frac1B\cdot\sum_{b=1}^B\left(\frac{\sum_{i=1}^n k_{DT,b}(\xsv,\xvi)y_i}{\sum_{j=1}^nk_{DT,b}(\xsv,\xvj)}\right)\\
&=\frac1B\cdot\sum_{i=1}^n\left(\sum_{b=1}^B \frac{k_{DT,b}(\xsv,\xvi)}{\sum_{j=1}^nk_{DT,b}(\xsv,\xvj)}\cdot y_i\right)
=:\frac1B\cdot\sum_{i=1}^nk_{RF}(\xsv,\xvi) y_i\\
&=\frac{\sum_{i=1}^nk_{RF}(\xsv,\xvi) y_i}{\sum_{j=1}^nk_{RF}(\xsv,\xvj)},
\end{aligned}
\end{equation*}
where $k_{RF}(\xsv,\xvi):=\sum_{b=1}^B \frac{k_{DT,b}(\xsv,\xvi)}{\sum_{i=1}^nk_{DT,b}(\xsv,\xvi)}=\sum_{b=1}^B \frac{k_{DT,b}(\xsv,\xvi)}{|R_b(\xsv)|}$ and 

\begin{equation*}
\begin{aligned}
\sum_{i=1}^nk_{RF}(\xsv,\xvi)
&=\sum_{i=1}^n\sum_{b=1}^B\left(\frac{k_{DT,b}(\xsv,\xvi)}{\sum_{j=1}^nk_{DT,b}(\xsv,\xvj)}\right)
=\sum_{b=1}^B\left(\frac{\sum_{i=1}^n k_{DT,b}(\xsv,\xvi)}{\sum_{j=1}^nk_{DT,b}(\xsv,\xvj)}\right)
=\sum_{b=1}^B 1
=B.
\end{aligned}
\end{equation*}

Generalization of Theorem~\ref{thm:fc_ens}:
\begin{thm}
\label{thm:fc_ens1}
Let $\{\X_b\}_{b=1}^B$ denote $B$ identically distributed (but not necessarily independent) training datasets, and let $k_b$ denote the kernel obtained from $\X_b$. Let $\rho_{bb'}$ denote the correlation between the normalized kernels $\kb_b$ and $\kb_{b'}$. Let $k_B:=\frac1B\cdot\sum_{b=1}^Bk_b$ denote the kernel of the ensemble model. Then,
\begin{equation*}
\label{eq:fc_ens1}
\FC\left(k_B\right)=1-q^2\cdot\left(1-\left(\FC(k_b)+\left(1-\frac1B\right)\cdot\Var(\kb_b)\cdot\left(1-\rho_{bb'}\right)\right)\right),
\end{equation*}
where $q\in\left[\frac{\min_{\xv}k(\xv,\xv)}{\max_{\xv}k(\xv,\xv)},\ \frac{\max_{\xv}k(\xv,\xv)}{\min_{\xv}k(\xv,\xv)}\right]$.
\end{thm}

\section{Complexity as a Function of Training Time}
\label{sec:comp_time}
Our empirical observation is that for non-constant-feature models, the EFC tends to increase with training time, which is intuitive: a model based on randomly selected parameters is likely to be less complex than one where parameters are intelligently chosen to explain the data. Additionally, in the early training stages, the model can focus on explaining simpler patterns in the data, which requires a less complex model. However, to explain the more complex parts of the data, the model's complexity must eventually increase with training.

In Figure~\ref{fig:comp_time}, we plot the training MSE and EFC as functions of training time for neural networks with 20 hidden units on MNIST and CIFAR, using the same setup as in Section~\ref{sec:dd} and Appendix~\ref{sec:exp_dets}. After an initial decrease, likely due to random initialization, the EFC increases during training.

\begin{figure}[t]
  \centering
  \begin{subfigure}{0.495\textwidth}
    \centering
    \includegraphics[width=\textwidth]{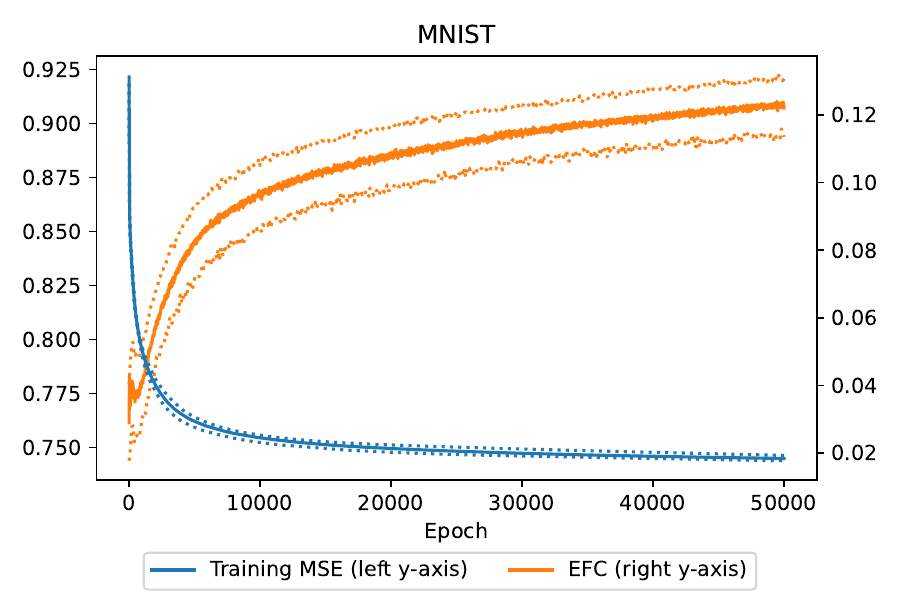}
  \end{subfigure}
  \hfill
  \begin{subfigure}{0.495\textwidth}
    \centering
    \includegraphics[width=\textwidth]{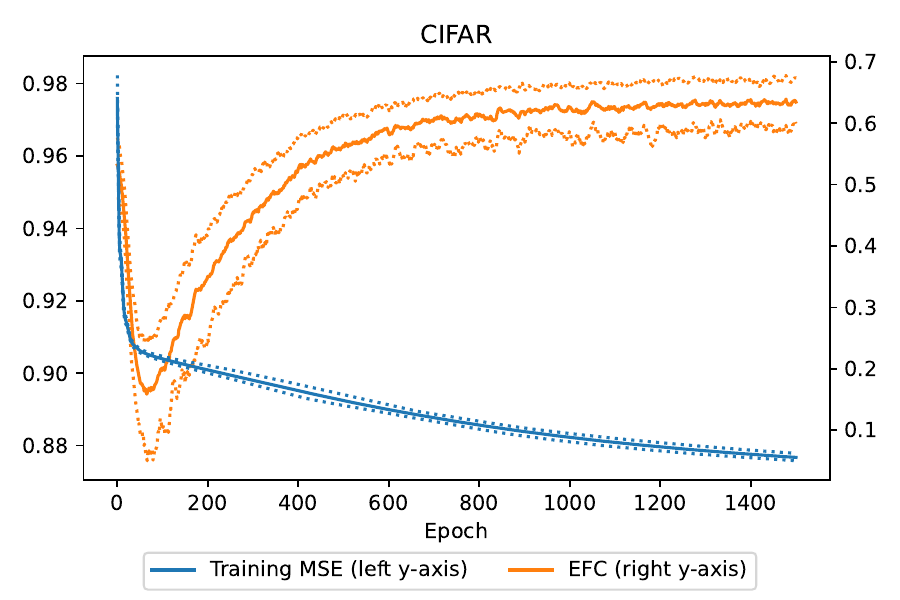}
  \end{subfigure}
  \caption{Training error and EFC during the first epochs of training neural networks on MNIST and CIFAR. The solid lines denote the mean over 10 random initializations, while the dotted lines denote the first and ninth deciles. After an initial decrease, likely due to random initialization, the EFC increases during training.}
  \label{fig:comp_time}
\end{figure}

\section{Proofs}
\label{sec:proofs}
For some of the proofs, we have used ChatGPT as a discussion partner.

\begin{proof}[Proof of Proposition~\ref{thm:lin_entr}]~\\
$$\EFC(k, \{\xvi\}_{i=1}^n):=1-\frac{1}{n^2-n}\cdot\sum_{\substack{i,j=1\\i\neq j}}^n\kb(\xvi,\xvj)^2,$$
where $\{\kb(\xvi,\xvj)\}_{\substack{i,j=1,i\neq j}}^n$ are the $n^2-n$ off-diagonal elements of $\Kb$. The sum of the squared off-diagonal elements can be written as $\|\I_n-\Kb\|_F^2$. Thus $\EFC(k,\{\xvi\}_{i=1}^n)=1-\frac1{n(n-1)}\cdot\|\I_n-\Kb\|_F^2$. Let $\{\lambda_i\}_{i=1}^n$ and $\{\lambdab_i\}_{i=1}^n$, where $\lambdab_i=\lambda_i/\sum_{j=1}^n\lambda_j=\lambda_i/n$ denote the eigenvalues and normalized eigenvalues of $\Kb$. Using that $\|\A\|_F^2=\Tr(\A^\top\A)$, that $\Kb$ is symmetric, that $\Tr(\Kb)=\sum_{i=1}^n\lambda_i=n$, that $\|\A\|_F^2=\sum_i\lambda_i(\A)^2$, and that $\sum_{i=1}^n\lambdab_i=1$, we obtain
\begin{equation*}
\begin{aligned}
&\EFC(k,\{\xvi\}_{i=1}^n)
=1-\frac1{n(n-1)}\cdot\|\I_n-\Kb\|_F^2
=1-\frac1{n(n-1)}\cdot\Tr\left(\I_n-2\Kb+\Kb^2\right)\\
&=1-\frac1{n(n-1)}\left(\Tr(\I_n)-2\cdot\Tr\left(\Kb\right)+\Tr\left(\Kb^2\right)\right)
=1-\frac1{n(n-1)}\left(n-2n+\frac nn\Tr\left(\Kb^2\right)\right)\\
&=1-\frac1{n-1}\left(-1+\frac 1n\|\Kb\|_F^2\right)
=\frac n{n-1}\left(1-\frac 1{n^2}\|\Kb\|_F^2\right)
=\frac{n}{n-1}\left(1-\sum_{i=1}^n\left(\frac{\lambda_i}{n}\right)^2\right)\\
&=\frac{n}{n-1}\left(1-\sum_{i=1}^n\lambdab_i^2\right)
=\frac{n}{n-1}\left(\sum_{i=1}^n\lambdab_i-\sum_{i=1}^n\lambdab_i^2\right)
=\frac{n}{n-1}\left(\sum_{i=1}^n\left(\lambdab_i-\lambdab_i^2\right)\right)\\
&=\frac{n}{n-1}\left(\sum_{i=1}^n\lambdab_i\left(1-\lambdab_i\right)\right).
\end{aligned}
\end{equation*}
\end{proof}

\begin{proof}[Proof of Theorem~\ref{thm:fc_pm_kern_fct}]~\\
For \textbf{part~\ref{thm:fc_pm_kern-pp}}, we first show the strict increase in $p$ and then that $\FC(k_p,p=0)=0$ and $\FC(k_p,p=\infty)=1$.

Let $\bm{\tilde{x}}=[\sqrt{c},\xv^\top]^\top$ and $\bm{\tilde{x}'}=[\sqrt{c},\xpv{^\top}]^\top$. Then

\begin{equation*}
\begin{aligned}
\kb(\xv,\xpv)^2=\frac{k_p(\xv,\xpv)^2}{k_p(\xv,\xv)\cdot k_p(\xpv,\xpv)}
=\left(\frac{(c+\xv^\top\xpv)^2}{(c+\|\xv\|_2^2)\cdot (c+\|\xpv\|_2^2)}\right)^p
=\left(\frac{(\bm{\tilde{x}}^\top\bm{\tilde{x}'})^2}{\|\bm{\tilde{x}}\|_2^2\cdot\|\bm{\tilde{x}'}\|_2^2}\right)^p.
\end{aligned}
\end{equation*}
According to the Cauchy-Schwarz inequality, 
$$\left(\frac{\bm{\tilde{x}}^\top\bm{\tilde{x}'}}{\|\bm{\tilde{x}}\|_2\cdot\|\bm{\tilde{x}'}\|_2}\right)^2\leq\left(\frac{\|\bm{\tilde{x}}\|_2\cdot\|\bm{\tilde{x}'}\|_2}{\|\bm{\tilde{x}}\|_2\cdot\|\bm{\tilde{x}'}\|_2}\right)^{2}=1,$$
with equality only if $\bm{\tilde{x}}$ and $\bm{\tilde{x}'}$ are parallel.  We first show that $\bm{\tilde{x}}$ and $\bm{\tilde{x}'}$ are parallel if and only if $\xv=\xpv$:
Let $x:=\|\xv\|_2$, $x':=\|\xpv\|_2$, and let $\theta$ denote the angle between $\xv$ and $\xpv$. Then 
$$\bm{\tilde{x}}\parallel\bm{\tilde{x}'}\iff (\bm{\tilde{x}}^\top\bm{\tilde{x}'})^2=\|\bm{\tilde{x}}\|_2^2\cdot\|\bm{\tilde{x}'}\|_2^2\iff (c+\cos\theta \cdot xx')^2=(c+x^2)\cdot(c+x'^2).$$
Solving for $x'$, we obtain
$$x'=\frac{x\cdot c\cdot \cos\theta\pm x\cdot\sqrt{(\cos^2\theta -1)c(c+x^2)}}{c+x^2(1-\cos^2\theta)},$$
where the expression in the square root is non-negative only if $\cos^2\theta-1\geq 0$, i.e.\ if $\cos\theta=\pm1$, in which case we obtain $x'=\pm x$ (where $x'=-x$ implies $x'=x=0$). This means that $\xv$ and $\xpv$ are parallel and of equal length, i.e., $\xv=\xpv$.

Now, let $\pi:=\Pr(\xv=\xpv)=\Pr(\bm{\tilde{x}} \parallel \bm{\tilde{x}'})=\Pr\left(\left(\frac{c+\xv^\top\xpv}{(c+\|\xv\|_2^2)(c+\|\xpv\|_2^2)}\right)^2=1\right)$. Then, since $\kb(\xv,\xv)=1$,

\begin{equation*}
\begin{aligned}
\E(\kb(\xv,\xpv))&=(1-\pi)\cdot \E\left(\kb(\xv,\xpv)|\xv\neq \xpv\right)+\pi\cdot\E\left(\kb(\xv,\xpv)|\xv=\xpv\right)\\
&=(1-\pi)\cdot \E\left(\kb(\xv,\xpv)|\xv\neq\xpv\right)+\pi.
\end{aligned}
\end{equation*}

When $\xv\neq\xpv$, and thus $\frac{(c+\xv^\top\xpv)^2}{(c+\|\xv\|_2^2)(c+\|\xpv\|_2^2)}<1$,

\begin{equation*}
\begin{aligned}
\left(\frac{(c+\xv^\top\xpv)^2}{(c+\|\xv\|_2^2)(c+\|\xpv\|_2^2)}\right)^{p_1}>\left(\frac{(c+\xv^\top\xpv)^2}{(c+\|\xv\|_2^2)(c+\|\xpv\|_2^2)}\right)^{p_2} \text{ for } p_1<p_2.
\end{aligned}
\end{equation*}
Since taking expectation preserves the inequality, additionally
\begin{equation*}
\begin{aligned}
&\E\left(\left(\left.\frac{(c+\xv^\top\xpv)^2}{(c+\|\xv\|_2^2)(c+\|\xpv\|_2^2)}\right)^{p_1}\right|\xv\neq \xpv\right)\\
&>\E\left(\left(\left.\frac{(c+\xv^\top\xpv)^2}{(c+\|\xv\|_2^2)(c+\|\xpv\|_2^2)}\right|\xv\neq \xpv\right)^{p_2}\right) \text{ for } p_1<p_2,
\end{aligned}
\end{equation*}
and, unless $\pi=1$,
\begin{equation*}
\begin{aligned}
\FC(k(p_1))=&1-(1-\pi)\cdot\E\left(\left.\left(\frac{(c+\xv^\top\xpv)^2}{(c+\|\xv\|_2^2)(c+\|\xpv\|_2^2)}\right)^{p_1}\right|\xv\neq \xpv\right) - \pi\\
< &1-(1-\pi)\cdot\E\left(\left.\left(\frac{(c+\xv^\top\xpv)^2}{(c+\|\xv\|_2^2)(c+\|\xpv\|_2^2)}\right)^{p_2}\right|\xv\neq \xpv\right) -\pi\\
=&\FC(k(p_2)), \text{ for } p_1<p_2.
\end{aligned}
\end{equation*}

For $p\to 0$, 

\begin{equation*}
\begin{aligned}
\FC(k_p,p=0)&=1-(1-\pi)\cdot \E\left(\left.\lim_{p\to 0}\left(\frac{(c+\xv^\top\xpv)^2}{(c+\|\xv\|_2^2)(c+\|\xpv\|_2^2)}\right)^p\right|\xv\neq \xpv\right)-\pi\cdot 1\\
&=1-(1-\pi)\cdot1-\pi=0.
\end{aligned}
\end{equation*}

For $p\to \infty$,\\

\begin{equation*}
\begin{aligned}
\FC(k_p,p=\infty)&=1-(1-\pi)\cdot\E\left(\left.\lim_{p\to\infty}\left(\frac{(c+\xv^\top\xpv)^2}{(c+\|\xv\|_2^2)(c+\|\xpv\|_2^2)}\right)^p\right|\xv\neq \xpv\right)-\pi\\
&=1-(1-\pi)\cdot 0-\pi\cdot 1=1-\pi.
\end{aligned}
\end{equation*}

~\\
~\\
For \textbf{part~\ref{thm:fc_pm_kern-ml}}, we first note that, since $k_M(\xv,\xv,l,\nu)=1$, 
$$\kb_M(\xv,\xpv,l,\nu)=\frac{k_M(\xv,\xpv,l,\nu)}{\sqrt{k_M(\xv,\xv,l,\nu)\cdot k_M(\xpv,\xpv,l,\nu)}}=k_M(\xv,\xpv,l,\nu).$$
We introduce $r:=\|\xv-\xpv\|_2$, $z(l):=\frac{\sqrt{2\nu}\,r}{l}$, and $C_\nu:=\frac{2^{1-\nu}}{\Gamma(\nu)}$.

We will first show the strict decrease in $l$, and then calculate $\FC(k_M,l=0)$ and $\FC(k_M,l=\infty)$.

Let $f(z):=z^\nu K_\nu(z)$. Then $k_M(l)=C_\nu \cdot f(z(l))$ and, according to the chain rule, 
$$\frac{\partial}{\partial l}\left(k_M(f(z(l)))^2\right)=2k_M(f)\cdot\frac{\partial k_M(f)}{\partial f}\cdot \frac{\partial f}{\partial z}\cdot\frac{\partial z}{\partial l},$$
where
$$\frac{\partial k_M(f)}{\partial f}=C_\nu,\ \frac{\partial f}{\partial z}=\frac{\partial}{\partial z}\left(z^\nu K_\nu(z)\right)=-z^\nu K_{\nu-1}(z),\text{ and }\frac{\partial z}{\partial l}=-\frac{\sqrt{2\nu}\,r}{l^2},$$
where the second equality is a property of the modified Bessel function of the second kind.
Thus,
\begin{equation*}
\begin{aligned}
\frac{\partial}{\partial l}\left(k_M(l)^2\right)
&=2C_\nu z^\nu K_\nu(z)\cdot C_\nu\cdot (-z^\nu K_{\nu-1}(z)) \cdot \frac{-\sqrt{2\nu}\,r}{l^2}\\
&=2C_\nu^2 \cdot z^{2\nu} \cdot K_\nu(z)\cdot K_{\nu-1}(z) \cdot \frac{\sqrt{2\nu}\,r}{l^2}\\
&=\frac{2^{3-2\nu}}{\Gamma(\nu)^2 \,l}\cdot \left(\frac{\sqrt{2\nu}\,r}{l}\right)^{2\nu+1} \cdot K_\nu\left(\frac{\sqrt{2\nu}\,r}{l}\right)\cdot K_{\nu-1}\left(\frac{\sqrt{2\nu}\,r}{l}\right).
\end{aligned}
\end{equation*}

Now, with $\pi:=\Pr(\xv=\xpv)$, 

\begin{equation*}
\begin{aligned}
\frac{\partial \FC(k_M)}{\partial l}=&-\E\left(\frac{\partial k_M^2}{\partial l}\right)
=-(1-\pi)\cdot\E\left(\left.\frac{\partial k_M^2}{\partial l}\right|\xv\neq \xpv\right)+ \pi\cdot \E\left(\left.\frac{\partial k_M^2}{\partial l}\right|\xv=\xpv\right)\\
=&-(1-\pi)\cdot\E\left(\left.\underbrace{\frac{2^{3-2\nu}}{\Gamma(\nu)^2 \,l}\cdot \left(\frac{\sqrt{2\nu}\,r}{l}\right)^{2\nu+1} \cdot K_\nu\left(\frac{\sqrt{2\nu}\,r}{l}\right)\cdot K_{\nu-1}\left(\frac{\sqrt{2\nu}\,r}{l}\right)}_{>0}\right|r>0\right)\\
&+ \pi\cdot 0<0.
\end{aligned}
\end{equation*}

For $l\to 0$, we use that when $r>0$, then $l\to 0$ implies $z\to \infty$, in which case $K_\nu(z)\sim \sqrt{\frac{\pi}{2z}}e^{-z}$, while for $r=0$, $k_M(\xv,\xv)=1$. Hence,

\begin{equation*}
\begin{aligned}
k_M(l=0)&=1-(1-\pi)\cdot\E\left(\left.\lim_{l\to 0} k_M(l)^2\right|r>0\right)-\pi\cdot 1\\
&=1-(1-\pi)\cdot\E\left(\lim_{z\to\infty}\left(\frac{2^{1-\nu}}{\Gamma(\nu)} z^\nu \cdot \sqrt{\frac{\pi}{2z}}e^{-z}\right)^2\right)-\pi\\
&=1-(1-\pi)\cdot 0^2 -\pi = 1-\pi.
\end{aligned}
\end{equation*}

For $l\to \infty$, we have $z\to 0$, in which case $K_\nu(z)\sim 2^{\nu-1}\Gamma(\nu)z^{-\nu}$. Hence,

\begin{equation*}
\begin{aligned}
k_M(l=\infty)&=1-\E\left(\lim_{l\to \infty} k_M(l)^2\right)
=1-\E\left(\lim_{z\to 0}\left(\frac{2^{1-\nu}}{\Gamma(\nu)} z^\nu \cdot 2^{\nu-1}\Gamma(\nu)z^{-\nu}\right)^2\right)\\
&=1-1^2=0.
\end{aligned}
\end{equation*}
\end{proof}

\begin{proof}[Proof of Theorem~\ref{thm:fc_pm_kern_data}]~\\
For \textbf{parts~\ref{thm:fc_pm_kern-pd} and \ref{thm:fc_pm_kern-ps}}, we start with the trivial cases $d=0$ and $\sigma=0$.

For $d=0$, we obtain $\xv=\xpv=0$, while for $\sigma=0$, we obtain $\xv=\xpv=\nv$. In both cases, $\xv^\top\xpv=\|\xv\|_2=\|\xpv\|_2=0$, and 

$$\E\left(\left(\frac{(c+\xv^\top\xpv)^2}{(c+\|\xv\|_2^2)(c+\|\xpv\|_2^2)}\right)^p\right)=\E\left(\left(\frac{(c+0)^2}{(c+0)(c+0)}\right)^p\right)=1,$$
and thus $\FC(k_P(\xv,\xpv,d=0))=\FC(k_P(\xv,\xpv,\sigma=0))=1-1=0$.

To investigate when $d$ and $\sigma$ go to infinity, we let $\xv=\sigma\cdot\bm{z}$ and $\xpv=\sigma\cdot\bm{z'}$, where $\bm{z},\bm{z'}\sim\N(0,\I_d)$.
Additionally, we let $\tilde{c}:=c/\sigma^2$, $s:=\|\bm{z}\|_2$, $s':=\|\bm{z'}\|_2$, and $\rho:=\bm{z}^\top\bm{z'}/\left(\|\bm{z}\|_2\cdot\|\bm{z'}\|_2\right)$. Then

\begin{equation*}
\begin{aligned}
\E\left(\left(\frac{(c+\xv^\top\xpv)^2}{(c+\|\xv\|_2^2)(c+\|\xpv\|_2^2)}\right)^p\right)
=\E\left(\left(\frac{(c/\sigma^2+\bm{z}^\top\bm{z'})^2}{(c/\sigma^2+\|\bm{z}\|_2^2)(c/\sigma^2+\|\bm{z'}\|_2^2)}\right)^p\right)\\
=\E\left(\left(\frac{(\tilde{c}+\rho ss')^2}{(\tilde{c}+s^2)(\tilde{c}+s'^2)}\right)^p\right)
=\E_{s,s'}\left(\frac1{(\tilde{c}+s^2)^p(\tilde{c}+s'^2)^p}\cdot\E_\rho\left((\tilde{c}+\rho ss')^{2p}\right)\right).
\end{aligned}
\end{equation*}
For fixed $s$ and $s'$, let us consider the quantity $\E_\rho\left((\tilde{c}+\rho ss')^{2p}\right)$:
\begin{equation*}
\begin{aligned}
\E_\rho\left((\tilde{c}+\rho ss')^{2p}\right)
=\E_\rho\left(\sum_{m=0}^{2p} {\binom{2p}{m}} \tilde{c}^{2p-m}(\rho ss')^m\right)
=\sum_{k=0}^{p} {\binom{2p}{2k}} \tilde{c}^{2p-2k}(ss')^{2k}\E_\rho\left(\rho^{2k}\right),
\end{aligned}
\end{equation*}
where we have used the fact that the odd moments of $\rho$ are 0 because $\rho$ is symmetric around 0.

To calculate the even moments, we need the distribution of $\rho^2=\left(\left(\frac{\bm{z}}{\|\zv\|_2}\right)^\top\left(\frac{\zpv}{\|\zpv\|_2}\right)\right)^2$. Since $\zv,\zpv\sim\N(\nv,\I_d)$, $\zv/\|\zv\|_2$ and $\zpv/\|\zpv\|_2$ are both spherically symmetric on the unit sphere. We can thus WLOG assume that $\zpv/\|\zpv\|_2=\bm{e_1}=(1,0,0,\dots)$, which implies that 
$$\rho^2=\left(\frac{\zv^\top\bm{e_1}}{\|\zv\|_2}\right)^2=\frac{z_1^2}{\|\zv\|_2^2}=\frac{z_1^2}{\sum_{i=1}^d z_i^2}=\frac{z_1^2}{z_1^2+\sum_{i=2}^dz_i^2}=\frac{\chi_1^2}{\chi_1^2+\chi_{d-1}^2},$$
where $\chi^2_1$ and $\chi^2_{d-1}$ are two independent chi square distributed random variables. We can now use the fact that $\chi_a^2/(\chi_a^2+\chi_b^2)$ is beta distributed with parameters $a/2$ and $b/2$, to obtain
$$\rho^2\sim\text{Beta}\left(\frac12,\frac{d-1}{2}\right).$$
The moments of a Beta$(a,b)$ distributed random variable are given by $\E(x^k)=\prod_{i=0}^{k-1}\frac{a+i}{a+b+i}$. Thus, in our case
$$\E\left(\rho^{2k}\right)=\E\left((\rho^2)^k\right)=\prod_{i=0}^{k-1}\frac{\frac12+i}{\frac12+\frac {d-1}2+i}=\prod_{i=0}^{k-1}\frac{1+2i}{d+2i}.$$

Additionally, 
\begin{equation*}
\begin{aligned}
\E_{s,s'}\left(\frac{(ss')^{2k}}{(\tilde{c}+s^2)^p(\tilde{c}+s'^2)^p}\right)
&=\E_s\left(\frac{s^{2k}}{(\tilde{c}+s^2)^p}\right)\cdot \E_{s'}\left(\frac{s'^{2k}}{(\tilde{c}+s'^2)^p}\right)
=\left(\E_s\left(\frac{(s^2)^k}{(\tilde{c}+s^2)^p}\right)\right)^2\\
&=\left(\E\left(\frac{(\chi^2_d)^k}{(\tilde{c}+\chi_d^2)^p}\right)\right)^2,
\end{aligned}
\end{equation*}
and thus,
$$\E\left(\left(\frac{(c+\xv^\top\xpv)^2}{(c+\|\xv\|_2^2)(c+\|\xpv\|_2^2)}\right)^p\right)
=\sum_{k=0}^{p} {\binom{2p}{2k}} \cdot \left(\frac{c}{\sigma^2}\right)^{2p-2k}\cdot \E\left(\frac{(\chi_d^2)^k}{(c/\sigma^2+\chi_d^2)^p}\right)^2\cdot\prod_{i=0}^{k-1}\frac{1+2i}{d+2i}.$$
As $\sigma^2\to\infty$, $\frac{(\chi_d^2)^k}{(c/\sigma^2+\chi_d^2)^p}\to (\chi_d^2)^{k-p}$ and $\left(\frac{c}{\sigma^2}\right)^{2p-2k}\to\begin{cases}0, k<p\\1, k=p\end{cases}$. Thus, only the last term in the sum survives and 
$$\lim_{\sigma\to\infty}\left(\E\left(\left(\frac{(c+\xv^\top\xpv)^2}{(c+\|\xv\|_2^2)(c+\|\xpv\|_2^2)}\right)^p\right)\right)=\prod_{i=0}^{p-1}\frac{1+2i}{d+2i}.$$
Hence $\FC(k_P,\sigma=\infty)=1-\prod_{i=0}^{p-1}\frac{1+2i}{d+2i}$.

As $d\to\infty$, $\frac{1+2i}{d+2i}\to 0$, so for $k>0$ each term in the sum goes to 0, while for $k=0$, 
$\E\left(\frac{(\chi_d^2)^0}{(c/\sigma^2+\chi_d^2)^p}\right)^2$
goes to $0$ as $d\to\infty$.

Thus,
$$\lim_{d\to\infty}\left(\E\left(\left(\frac{(c+\xv^\top\xpv)^2}{(c+\|\xv\|_2^2)(c+\|\xpv\|_2^2)}\right)^p\right)\right)=0,$$
and $\FC(k_P,d=\infty)=1-0=1$.

~\\

For \textbf{parts~\ref{thm:fc_pm_kern-md} and \ref{thm:fc_pm_kern-ms}}, we must show that $k_M(d+1)^2-k_M(d)^2\stackrel{\text{a.s.}}<0$, $\FC(k_M,d=0)=0$, $\FC(k_M,d=\infty)=1$, $\frac{\partial k_M^2}{\partial \sigma}\stackrel{\text{a.s.}}<0$, $\FC(k_M,\sigma=0)=0$, and $\FC(k_M,\sigma=\infty)=1$.
We let $f(z):=z^\nu K_\nu(z)$ so that $k_M(l)=C_\nu \cdot f(z(l))$.

Showing that $k_M(d+1)^2-k_M(d)^2\stackrel{\text{a.s.}}<0$:\\
According to the chain rule, 
$$\frac{\partial k_M(f(z))^2}{\partial z}=2k_M(f)\cdot\frac{\partial k_M(f)}{\partial f}\cdot \frac{\partial f}{\partial z},$$
where 

$$\frac{\partial k_M(f)}{\partial f}=C_\nu,\text{ and } \frac{\partial f}{\partial z}=\frac{\partial}{\partial z}\left(z^\nu K_\nu(z)\right)=-z^\nu K_{\nu-1}(z).$$

Thus,
\begin{equation}
\label{eq:dkmdz}
\frac{\partial k_M(f(z))^2}{\partial z}
=2C_\nu z^\nu K_\nu(z)\cdot C_\nu\cdot (-z^\nu K_{\nu-1}(z)) 
=-2C_\nu^2 \cdot z^{2\nu} \cdot K_\nu(z)\cdot K_{\nu-1}(z) \stackrel{\text{a.s.}}<0.
\end{equation}

Since $\chi^2_{d+1}=\chi^2_d+\zeta^2$, where $\zeta\sim\N(0,1)$; $\chi^2_d$, and thus $\sqrt{\chi^2_d}$, almost surely increases in $d$. 

Consequently, $z(d+1)=\frac{2\sqrt{\nu}\sigma\sqrt{\chi^2_{d+1}}}{l}\stackrel{a.s.}>z(d)=\frac{2\sqrt{\nu}\sigma\sqrt{\chi^2_d}}{l}$.
This, in combination with Equation \ref{eq:dkmdz}, shows that that $k_M(d+1)^2-k_M(d)^2\stackrel{\text{a.s.}}<0$, and thus $\FC(k_M,d+1)-\FC(k_M,d)=-\E\left(k_M(d+1)^2-k_M(d)^2\right)>0$.

~\\
The case $d\to0$:\\
For $d\to0$, $\xv,\xpv\to 0$, $z\to 0$, and $K_\nu(t)\sim 2^{\nu-1}\Gamma(\nu)t^{-\nu}$. Hence,
$$k_M(\xv,\xpv,d\to0)\sim\frac{2^{1-\nu}}{\Gamma(\nu)} z^\nu \cdot 2^{\nu-1}\Gamma(\nu)z^{-\nu}=1,$$
and $\FC(k_M(\xv,\xpv,d=0))=1-1^2=0$.

~\\
The case $d\to\infty$:\\
As $d\to \infty$, by the law of large numbers, $\chi^2_d/d\stackrel{\text{a.s.}}\to 1$, and thus a.s.\ $z\sim \frac{2\sqrt{\nu}\sigma \sqrt{d}}{l}$, which goes to infinity with $d$. As $z\to\infty$, $K_\nu(z)\sim \sqrt{\frac{\pi}{2z}}e^{-z}\to 0$. Hence $\FC(k_M,d=\infty)=1-0^2=1$.

~\\
Showing that $\frac{\partial k_M^2}{\partial \sigma}\stackrel{\text{a.s.}}<0$:\\
According to the chain rule,
$$\frac{\partial}{\partial \sigma}\left(k_M(f(z(\sigma)))^2\right)=2k_M(f)\cdot\frac{\partial k_M(f)}{\partial f}\cdot \frac{\partial f}{\partial z}\cdot\frac{\partial z}{\partial \sigma},$$
where
$$\frac{\partial k_M(f)}{\partial f}=C_\nu,\ \frac{\partial f}{\partial z}=\frac{\partial}{\partial z}\left(z^\nu K_\nu(z)\right)=-z^\nu K_{\nu-1}(z),\text{ and }\frac{\partial z}{\partial \sigma}=\frac{2\sqrt{\nu}\sqrt{\chi^2_d}}{l}.$$

Thus,
\begin{equation*}
\begin{aligned}
\frac{\partial}{\partial \sigma}\left(k_M(\sigma)^2\right)
&=2C_\nu z^\nu K_\nu(z)\cdot C_\nu\cdot (-z^\nu K_{\nu-1}(z)) \cdot \frac{2\sqrt{\nu}\sqrt{\chi^2_d}}{l}\\
&=-2C_\nu^2 \cdot z^{2\nu} \cdot K_\nu(z)\cdot K_{\nu-1}(z) \cdot \frac{2\sqrt{\nu}\sqrt{\chi^2_d}}{l}\stackrel{\text{a.s.}}<0,
\end{aligned}
\end{equation*}
and $\frac{\partial \FC(k_M)}{\partial \sigma}=-\E\left(\frac{\partial k_M^2}{\partial \sigma}\right)>0$.

~\\
The case $\sigma\to0$:\\
For $\sigma\to0$, $\xv,\xpv\to\nv$, $z\to 0$, and $K_\nu(t)\sim 2^{\nu-1}\Gamma(\nu)t^{-\nu}$. Hence,
$$k_M(\xv,\xpv,\sigma\to0)\sim\frac{2^{1-\nu}}{\Gamma(\nu)} z^\nu \cdot 2^{\nu-1}\Gamma(\nu)z^{-\nu}=1,$$
and $\FC(k_M(\xv,\xpv,\sigma=0))=1-1^2=0$.

~\\
The case $\sigma\to\infty$:\\
As $\sigma\to \infty$, $z\to \infty$, and $K_\nu(z)\sim \sqrt{\frac{\pi}{2z}}e^{-z}\to 0$. Hence $\FC(k_M,\sigma=\infty)=1-0^2=1$.

\end{proof}

\begin{proof}[Proof of Proposition~\ref{thm:knn}]~\\
For part \emph{(a)}, we let $\Knntr$ denote the kernel matrix on training data. Then, each row in $\Knntr$ contains exactly $\kappa$ ones and $n-\kappa$ zeros, and since the closest neighbor of $\xv$ is $\xv$ itself, the diagonal elements are always one. Since the diagonal is one, $k(\xv,\xv)=1$, and thus $\kb(\xv,\xpv)=k(\xv,\xpv)$, i.e., $\Knntr=\Kb_{kNN}$ is already normalized. Since all rows are identical except for perturbations (they all contain exactly $\kappa$ ones and $n-\kappa$ zeros), they contribute equally to the empirical mean. This means that we only need to consider one row. Each row has $\kappa-1$ off-diagonal ones, and $n-1$ off diagonal-elements, and hence the empirical estimate of $\E(\kb(\xv,\xpv)^2)$ is given by $\frac{\kappa-1}{n-1}$, and $\EFC(k_{kNN},\X)=1-\frac{\kappa-1}{n-1}$.

~\\
For part $\emph{(b)}$, we let $\Kttr$ denote the kernel matrix on training data. We can, without loss of generality, sort the training data so that observations that belong to the same region $R$ are adjacent to each other, i.e., $\bm{x_1},\ \bm{x_2},\dots{\bm{x_a}}\in R_1,\ \bm{x_{a+1}},\ \bm{x_{a+2}},\dots {\bm{x_b}} \in R_2\dots$. 
$\Kttr$ then becomes block-diagonal with blocks of matrices of only ones, i.e., $\Kttr=\diag\left(\bm{1_a}\bm{1_a}^\top,\ \bm{1_b}\bm{1_b}^\top,\dots\right)$, where $\bm{1_a}\bm{1_a}^\top\in \R^{a\times a},\ \bm{1_b}\bm{1_b}^\top\in \R^{(b-a)\times (b-a)},\dots$ are matrices of only ones.
Since the diagonal is one, $k(\xv,\xv)=1$, and thus $\kb(\xv,\xpv)=k(\xv,\xpv)$, i.e., $\Kttr=\Kb_{DT}$ is already normalized.

Adding a split corresponds to splitting region $R_j$ into the two regions $R_{j_1}$ and $R_{j_2}$, and thus $\Kttr$ changes from $\diag\left(\dots,\ \bm{1_j}\bm{1_j}^\top,\ \dots\right)$ to $\diag\left(\dots,\ \bm{1_{j_1}}\bm{1_{j_1}}^\top,\ \bm{1_{j_2}}\bm{1_{j_2}}^\top,\ \dots\right)$. Since $\diag\left(\bm{1_{j_1}}\bm{1_{j_1}}^\top,\ \bm{1_{j_2}}\bm{1_{j_2}}^\top\right)$ is a block-diagonal matrix with both ones and zeros, it contains fewer ones than $\bm{1_j}\bm{1_j}^\top$. Thus, after the added split, $\Kttr$ contains more zeros than before the split, which means that the average of the off-diagonal elements is smaller, and the complexity is larger.

When there are no splits (one region), $\Kttr=\bm{1_n}{\bm{1_n}}^\top$, i.e., all off-diagonal elements are 1, and $\EFC(k_{DT},\X)=1-1=0$.

When each $\xvi$ has its own region, $\Kttr=\I_n$, i.e., all off-diagonal elements are 0, and $\EFC(k_{DT},\X)=1-0=1$.
\end{proof}

\begin{proof}[Proof of Theorems~\ref{thm:fc_ens} and \ref{thm:fc_ens1}]~\\
Since Theorem~\ref{thm:fc_ens} is a special case of Theorem~\ref{thm:fc_ens1}, with $q=1$, we prove the more general case with an arbitrary $q$.

We first show that 

\begin{equation*}
\begin{aligned}
1-q_{\max}^2\cdot\left(1-\FC\left(\frac1B\cdot\sum_{b=1}^B\kb_b\right)\right)
\leq\FC\left(\frac1B\cdot\sum_{b=1}^Bk_b\right)\leq
1-q_{\min}^2\cdot\left(1-\FC\left(\frac1B\cdot\sum_{b=1}^B\kb_b\right)\right)
\end{aligned}
\end{equation*}
where $\FC\left(\frac1B\cdot\sum_{b=1}^Bk_b\right)$ is the complexity if we first aggregate and then normalize the kernels, and $\FC\left(\frac1B\cdot\sum_{b=1}^B\kb_b\right)$ is the complexity if we first normalize and then aggregate:
\begin{equation*}
\begin{aligned}
\FC\left(\frac1B\cdot\sum_{b=1}^Bk_b\right)&=1-\E\left(\left(\frac{\frac1B\sum_{b=1}^Bk(\xv_b,\xpv_b)}{\sqrt{\frac1B\sum_{b=1}^Bk(\xv_b,\xv_b)}\cdot\sqrt{\frac1B\sum_{b=1}^Bk(\xpv_b,\xpv_b)}}\right)^2\right)\\
\FC\left(\frac1B\cdot\sum_{b=1}^B\kb_b\right)&
=1-\E\left(\left(\frac1B\sum_{b=1}^B\frac{k(\xv_b,\xpv_b)}{\sqrt{k(\xv_b,\xv_b)}\cdot\sqrt{k(\xpv_b,\xpv_b)}}\right)^2\right)\\
q_{\min}&=\frac{\min_{\xv}k(\xv,\xv)}{\max_{\xv}k(\xv,\xv)}\\
q_{\max}&=\frac{\max_{\xv}k(\xv,\xv)}{\min_{\xv}k(\xv,\xv)}.
\end{aligned}
\end{equation*}

\begin{align*}
&\FC\left(\frac1B\cdot\sum_{b=1}^Bk_b\right)
=1-\E\left(\left(\frac{\frac1B\sum_{b=1}^Bk(\xv_{b},\xpv_{b})}{\sqrt{\frac1B\sum_{b=1}^Bk(\xv_{b},\xv_{b})}\cdot\sqrt{\frac1B\sum_{b=1}^Bk(\xpv_{b},\xpv_{b})}}\right)^2\right)\\
&\scalebox{0.95}{$\displaystyle =1-\E\left(\left(\frac1B\sum_{b=1}^B\frac{k(\xv_{b},\xpv_{b})}{\sqrt{k(\xv_b,\xv_b)}\cdot\sqrt{k(\xpv_b,\xpv_b)}}\cdot\frac{\sqrt{k(\xv_b,\xv_b)}}{\sqrt{\frac1B\sum_{b'=1}^Bk(\xv_{b'},\xv_{b'})}}\cdot\frac{\sqrt{k(\xpv_b,\xpv_b)}}{\sqrt{\frac1B\sum_{b'=1}^Bk(\xpv_{b'},\xpv_{b'})}}\right)^2\right)$}\\
&=1-\E\left(\left(\frac1B\sum_{b=1}^B\kb(\xv_b,\xpv_b)\cdot\frac{\sqrt{k(\xv_b,\xv_b)}}{\sqrt{\frac1B\sum_{b'=1}^Bk(\xv_{b'},\xv_{b'})}}\cdot\frac{\sqrt{k(\xpv_b,\xpv_b)}}{\sqrt{\frac1B\sum_{b'=1}^Bk(\xpv_{b'},\xpv_{b'})}}\right)^2\right)\\
&\leq1-\E\left(\left(\frac1B\sum_{b=1}^B\kb(\xv_b,\xpv_b)\cdot\sqrt{q_{\min}}\cdot\sqrt{q_{\min}}\right)^2\right)\\
&=1-\E\left(\left(\frac1B\sum_{b=1}^B\kb(\xv_b,\xpv_b)\right)^2\right)\cdot q_{\min}^2
=1- q_{\min}^2\cdot\left(1-1+\E\left(\left(\frac1B\sum_{b=1}^B\kb(\xv_b,\xpv_b)\right)^2\right)\right)\\
&=1- q_{\min}^2\cdot\left(1-\FC\left(\frac1B\sum_{b=1}^B\kb_b\right)\right).
\end{align*}
and, similarly,
\begin{equation*}
\begin{aligned}
&\FC\left(\frac1B\cdot\sum_{b=1}^Bk_b\right)\\
&=1-\E\left(\left(\frac1B\sum_{b=1}^B\kb(\xv,\xpv)\cdot\frac{\sqrt{k(\xv_b,\xv_b)}}{\sqrt{\frac1B\sum_{b'=1}^Bk(\xv_{b'},\xv_{b'})}}\cdot\frac{\sqrt{k(\xpv_b,\xpv_b)}}{\sqrt{\frac1B\sum_{b'=1}^Bk(\xpv_{b'},\xpv_{b'})}}\right)^2\right)\\
&\geq1-\E\left(\left(\frac1B\sum_{b=1}^B\kb(\xv_b,\xpv_b)\cdot\sqrt{q_{\max}}\cdot\sqrt{q_{\max}}\right)^2\right)
=1-q_{\max}^2\cdot\left(1-\FC\left(\frac1B\sum_{b=1}^B\kb_b\right)\right).
\end{aligned}
\end{equation*}

We now show that
$$\FC\left(\frac1B\cdot\sum_{b=1}^B\kb_b\right)=\FC\left(\kb_b\right)+\left(1-\frac1B\right)\cdot\Var(\kb_b)\cdot\left(1-\rho_{bb'}\right).$$
\begin{equation*}
\begin{aligned}
&\FC\left(\frac1B\cdot\sum_{b=1}^B\kb_b\right)=1-\E\left(\left(\frac1B\cdot\sum_{b=1}^B\kb_b\right)^2\right)\\
&=1-\frac1{B^2}\cdot\E\left(\sum_{b=1}^B\kb_b^2+2\cdot\sum_{1\leq b < b'}\kb_{b}\cdot\kb_{b'}\right)\\
&=1-\frac1{B^2}\cdot\left(\sum_{b=1}^B\E\left(\kb_b^2\right)+2\cdot\sum_{1\leq b<b'}\E\left(\kb_{b}\cdot \kb_{b'}\right)\right)\\
&=1-\frac1{B^2}\cdot\left(\sum_{b=1}^B\E\left(\kb_b^2\right)+2\cdot\sum_{1\leq b<b'}\left(\Cov\left(\kb_{b},\kb_{b'}\right)+\E\left(\kb_{b}\right)\cdot\E\left(\kb_{b'}\right)\right)\right)\\
&=1-\frac{B}{B^2}\cdot\E\left(\kb_b^2\right)-\frac{B\cdot(B-1)}{B^2}\cdot\left(\Cov\left(\kb_{b},\kb_{b'}\right)+\E\left(\kb_{b}\right)\cdot\E\left(\kb_{b'}\right)\right)\\
&=1-\E\left(\kb_b^2\right)+\E\left(\kb_b^2\right)-\frac1B\cdot\E\left(\kb_b^2\right)-\frac{B-1}{B}\cdot\left(\rho_{bb'}\cdot\sqrt{\Var(\kb_b)\Var(\kb_{b'})}+\E\left(\kb_{b}\right)\cdot\E\left(\kb_{b'}\right)\right)\\
&=1-\E\left(\kb_b^2\right)+\left(1-\frac1B\right)\cdot\left(\E\left(\kb_b^2\right)-\E\left(\kb_{b}\right)^2-\rho_{bb'}\cdot\Var(\kb_b)\right)\\
&=1-\E\left(\kb_b^2\right)+\left(1-\frac1B\right)\cdot\Var(\kb_b)\cdot\left(1-\rho_{bb'}\right)\\
&=\FC\left(k_b\right)+\left(1-\frac1B\right)\cdot\Var(\kb_b)\cdot\left(1-\rho_{bb'}\right).
\end{aligned}
\end{equation*}

Putting it together, we obtain
$$\FC\left(\frac1B\cdot\sum_{b=1}^Bk_b\right)=1-q^2\cdot\left(1-\left(\FC(k_b)+\left(1-\frac1B\right)\cdot\Var(\kb_b)\cdot\left(1-\rho_{bb'}\right)\right)\right)$$
where $q\in\left[q_{\min},\ q_{\max}\right]=\left[\frac{\min_{\xv}k(\xv,\xv)}{\max_{\xv}k(\xv,\xv)},\ \frac{\max_{\xv}k(\xv,\xv)}{\min_{\xv}k(\xv,\xv)}\right]$.

For $q=1$, we obtain
$$\FC\left(\frac1B\cdot\sum_{b=1}^Bk_b\right)=\FC(k_b)+\left(1-\frac1B\right)\cdot\Var(\kb_b)\cdot\left(1-\rho_{bb'}\right).$$

\end{proof}

\begin{proof}[Proof of Footnote~\ref{fn:mv_const}]~\\
Let 
$$\thetatv:=\frac{\thetahv}{M-\|\thetahv\|}\iff \thetahv=(M-\|\thetahv\|)\cdot\thetatv\implies\|\thetatv\|=\frac{\|\thetahv\|}{M-\|\thetahv\|}\in [0,\infty],$$
and

$$\|\thetahv\|=(M-\|\thetahv\|)\cdot\|\thetatv\|\iff \|\thetahv\|=M\cdot\frac{\|\thetatv\|}{1+\|\thetatv\|}\in [0,M].$$

Additionally, 
$$\thetahv=(M-\|\thetahv\|)\cdot\thetatv=\left(M-M\cdot\frac{\|\thetatv\|}{1+\|\thetatv\|}\right)\thetatv=\frac{M}{1+\|\thetatv\|}\cdot\thetatv.$$

Thus,
$$\thetahv^\top\phiv(\xv)=\frac{M}{1+\|\thetatv\|}\cdot\thetatv^\top\phiv(\xv).$$
\end{proof}

\begin{proof}[Proof of Footnote~\ref{fn:enp_mse}]~\\
When $\Ss$ is symmetric with eigenvalues in $[0,1]$, $\I_n-\Ss$ is positive semi-definite, so the nuclear norm and the trace coincide, i.e., $\|\I_n-\Ss\|_*=\Tr(\I_n-\Ss)$, and thus
\begin{equation*}
\begin{aligned}
\text{MSE}&=\frac1n\cdot\|\yv-\fhv\|_2^2=\frac1n\cdot\|\yv-\Ss\yv\|_2^2=\frac1n\cdot\|(\I_n-\Ss)\yv\|_2^2\leq \frac1n\cdot\|\I_n-\Ss\|_2^2\cdot\|\yv\|_2^2\\
&\leq\frac1n\cdot\|\I_n-\Ss\|_*^2\cdot\|\yv\|_2^2=\frac1n\cdot\left(\Tr(\I_n-\Ss)\right)^2\cdot\|\yv\|_2^2=(n-\text{ENP})^2\cdot\frac1n\cdot\|\yv\|_2^2.
\end{aligned}
\end{equation*}
\end{proof}

\begin{proof}[Proof of Footnote~\ref{fn:corr_k}]~\\
When $\fh(\xv)=\thetahv^\top\phiv(\xv)$ and $\Cov(\thetahv)=\I_p$,
\begin{equation*}
\begin{aligned}
&\text{\normalfont corr}_{\thetahv}\left(\fh(\xv),\fh(\xpv)\right)
=\frac{\Cov_{\thetahv}(\thetahv^\top\phiv(\xv),\thetahv^\top\phiv(\xpv))}{\sqrt{\Var_{\thetahv}(\thetahv^\top\phiv(\xv))\cdot \Var_{\thetahv}(\thetahv^\top\phiv(\xpv))}}\\
&=\frac{\phiv(\xv)^\top\Cov_{\thetahv}(\thetahv)\phiv(\xpv)}{\sqrt{\phiv(\xv)^\top\Cov_{\thetahv}(\thetahv)\phiv(\xv)\cdot\phiv(\xpv)^\top\Cov_{\thetahv}(\thetahv)\phiv(\xpv)}}
=\frac{\phiv(\xv)^\top\phiv(\xpv)}{\sqrt{\phiv(\xv)^\top\phiv(\xv)\cdot\phiv(\xpv)^\top\phiv(\xpv)}}\\
&=\kb(\xv,\xpv).
\end{aligned}
\end{equation*}
\end{proof}

\begin{proof}[Proof of Proposition~\ref{thm:gp}]~\\
We give the proof for the more general GP formulation, which includes a non-zero prior mean.
In this case,
\begin{equation*}
\begin{aligned}
&\fhv\sim\N(\muv,\K)\\
&\fhv_{\po}\sim\N\left(\muv+\K_t^\top\K_{tt\lambda}^{-1}(\yv-\muv_t),\K-\K_t^\top\K_{tt\lambda}^{-1}\K_t\right),
\end{aligned}
\end{equation*}
where $\muv\in\R^n$ denotes the mean vector of the GP, and $\muv_t\in\R^{n_t}$ denotes the mean vector on the training data.
First, we want to show that 
$$\fhv\sim\N(\muv,\K)\iff \fh(\xv)=\mu(\xv)+\zv^\top\phiv(\xv),$$
i.e.\ that the $n$ elements of $\fhv$, where $\fh_i=\mu_i+\zv^\top\bm{\varphi_i}$ (or, equivalently, $\fhv=\muv+\Ph\zv$), are normally distributed with mean $\muv$ and covariance $\K$.

Since $\zv$ follows a normal distribution, so does $\muv+\Ph\zv$. We just need to verify that it has the correct mean and covariance:
\begin{equation*}
\begin{aligned}
\E(\fhv)&=\muv+\Ph\E(\zv)=\muv+\Ph\nv=\muv\\
\Cov(\fhv)&=\Ph\Cov(\zv)\Ph^\top=\Ph\I_p\Ph^\top=\Ph\Ph^\top=\K.
\end{aligned}
\end{equation*}

For the posterior, we want to show that 
$$\fhv_{\po}\sim\N\left(\muv+\K_t^\top\K_{tt\lambda}^{-1}(\yv-\muv_t),\ \K-\K_t^\top\K_{tt\lambda}^{-1}\K_t\right) \iff \fh_{\po}(\xv) =\mu(\xv)+\thetahv^\top\phiv(\xv),$$
i.e., that the $n$ elements of 
$$\fhv_{\po}=\muv+\Ph\thetahv = \muv+\Ph\left(\Ph_t^\top\K_{tt\lambda}^{-1}(\yv-\muv_t)+\sqrt{\I_p-\Ph_t^\top\K_{tt\lambda}^{-1}\Ph_t}\cdot\zv\right)$$
are normally distributed with mean $\muv+\K_t^\top\K_{tt\lambda}^{-1}(\yv-\muv_t)$ and covariance $\K-\K_t^\top\K_{tt\lambda}^{-1}\K_t$.
Since $\zv$ is normally distributed, so is $\muv+\Ph\thetahv$. The mean and covariance are given by
\begin{equation*}
\begin{aligned}
\E(\fhv_{\po})&=\muv+\Ph\left(\Ph_t^\top\K_{tt\lambda}^{-1}(\yv-\muv_t)+\sqrt{\I_p-\Ph_t^\top\K_{tt\lambda}^{-1}\Ph_t}\cdot\E(\zv)\right)\\
&=\muv+\Ph\Ph_t^\top\K_{tt\lambda}^{-1}(\yv-\muv_t)
=\muv+\K_t^\top\K_{tt\lambda}^{-1}(\yv-\muv_t)\\
\Cov(\fhv_{\po})&=\Ph\Cov\left(\Ph_t^\top\K_{tt\lambda}^{-1}(\yv-\muv_t)+\sqrt{\I_p-\Ph_t^\top\K_{tt\lambda}^{-1}\Ph_t}\cdot\zv\right)\Ph^\top\\
&=\Ph\sqrt{\I_p-\Ph_t^\top\K_{tt\lambda}^{-1}\Ph_t}\E(\zv\zv^\top)\sqrt{\I_p-\Ph_t^\top\K_{tt\lambda}^{-1}\Ph_t}\Ph^\top
=\Ph\left(\I_p-\Ph_t^\top\K_{tt\lambda}^{-1}\Ph_t\right)\Ph^\top\\
&=\K-\K_t^\top\K_{tt\lambda}^{-1}\K_t.
\end{aligned}
\end{equation*}

We finally show the equivalent form of $\thetahv$: By the matrix inversion lemma,
$$\Ph_t^\top\K_{tt\lambda}^{-1}=\Ph_t^\top(\Ph_t\Ph_t^\top+\lambda\I_{n_t})^{-1}=(\Ph_t^\top\Ph_t+\lambda\I_p)^{-1}\Ph_t^\top.$$
Thus,
\begin{equation*}
\begin{aligned}
&\sqrt{\I_p-\Ph_t^\top\K_{tt\lambda}^{-1}\Ph_t}
=\sqrt{\I_p-(\Ph_t^\top\Ph_t+\lambda\I_p)^{-1}\Ph_t^\top\Ph_t}\\
&=\sqrt{(\Ph_t^\top\Ph_t+\lambda\I_p)^{-1}(\Ph_t^\top\Ph_t+\lambda\I_p-\Ph_t^\top\Ph_t)}
=\sqrt{(\Ph_t^\top\Ph_t+\lambda\I_p)^{-1}\cdot\lambda}\\
&=\sqrt{(1/\lambda\cdot\Ph_t^\top\Ph_t+\I_p)^{-1}}
=\left(1/\lambda\cdot\Ph_t^\top\Ph_t+\I_p\right)^{-1/2}
\end{aligned}
\end{equation*}
and
\begin{equation*}
\begin{aligned}
\thetahv&=\Ph_t^\top\K_{tt\lambda}^{-1}(\yv-\muv_t)+\sqrt{\I_p-\Ph_t^\top\K_{tt\lambda}^{-1}\Ph_t}\cdot\zv\\
&=(\Ph_t^\top\Ph_t+\lambda\I_p)^{-1}\Ph_t^\top(\yv-\muv_t)+\left(1/\lambda\cdot\Ph_t^\top\Ph_t+\I_p\right)^{-1/2}\cdot\zv.
\end{aligned}
\end{equation*}

\end{proof}

\begin{proof}[Proof of Proposition~\ref{thm:genp-v}]

\begin{equation*}
\begin{aligned}
&\E\left(\frac1{\ns}\cdot\|\ysv-\fhsv\|_2^2\right)
=\frac1{\ns}\cdot\E\left((\ysv-\Sss\yv)^\top(\ysv-\Sss\yv)\right)\\
&=\frac1{\ns}\cdot\left(\E(\ysv{^\top}\ysv)-2\E(\ysv){^\top}\Sss\E(\yv)+\E(\yv^\top\Sss{^\top}\Sss\yv)\right)\\
&=\frac1{\ns}\cdot\left(\Tr(\E(\ysv\ysv{^\top}))-0+\Tr(\Sss{^\top}\Sss\E(\yv\yv^\top))\right)
=\frac{\sigma^2_y}{\ns}\cdot\left(\Tr(\I_{\ns})+\Tr(\Sss{^\top}\Sss)\right)\\
&=\sigma^2_y\cdot\left(1+\frac1{\ns}\Tr(\Sss{^\top}\Sss)\right)
=\sigma^2_y\cdot\left(1+\frac1n\cdot\text{GENP-V}\right).
\end{aligned}
\end{equation*}
\end{proof}


\end{document}